\documentclass{article}

\usepackage{microtype}
\usepackage{graphicx}
\usepackage{subfigure}
\usepackage{booktabs} 

\usepackage{hyperref}

\usepackage[accepted]{icml2023}

\usepackage{amsmath}
\usepackage{amssymb}
\usepackage{mathtools}
\usepackage{amsthm}

\usepackage[capitalize,noabbrev]{cleveref}

\theoremstyle{plain}
\newtheorem{theorem}{Theorem}[section]

\newtheorem{lemma}[theorem]{Lemma}

\theoremstyle{definition}
\newtheorem{definition}[theorem]{Definition}

\theoremstyle{remark}
\newtheorem{remark}[theorem]{Remark}

\usepackage[textsize=tiny]{todonotes}

\usepackage{amssymb}
\usepackage{amsthm}
\usepackage{amsmath}
\usepackage{transparent}
\usepackage{multirow}
\usepackage{booktabs}

\newcommand{\tr}{\transparent{0.6}}
\usepackage{mathtools}

\usepackage{graphicx}
\usepackage{algorithm}
\usepackage{algorithmic}

\usepackage{tcolorbox}
\newtcolorbox{mybox}[3][]
{
  colframe = #2!25,
  colback  = #2!10,
  coltitle = #2!20!black,  
  title    = {#3},
  #1,
}

\usepackage{listings}
\usepackage{xcolor}
\definecolor{codegreen}{rgb}{0,0.6,0}
\definecolor{codegray}{rgb}{0.5,0.5,0.5}
\definecolor{codepurple}{rgb}{0.58,0,0.82}
\definecolor{backcolour}{rgb}{0.95,0.95,0.92}

\lstdefinestyle{mystyle}{
    backgroundcolor=\color{backcolour},   
    commentstyle=\color{codegreen},
    keywordstyle=\color{magenta},
    numberstyle=\tiny\color{codegray},
    stringstyle=\color{codepurple},
    basicstyle=\ttfamily\footnotesize,
    breakatwhitespace=false,         
    breaklines=true,                 
    captionpos=b,                    
    keepspaces=true,                 
    numbers=left,                    
    numbersep=5pt,                  
    showspaces=false,                
    showstringspaces=false,
    showtabs=false,                  
    tabsize=2
}

\definecolor{blu}{RGB}{0, 0, 230} 

\icmltitlerunning{One-shot Federated Learning with Foundation Models}

\begin{document}

\twocolumn[
\icmltitle{Parametric Feature Transfer: One-shot Federated Learning with Foundation Models}



\icmlsetsymbol{equal}{*}

\begin{icmlauthorlist}
\icmlauthor{Mahdi Beitollahi}{yyy}
\icmlauthor{Alex Bie}{yyy}
\icmlauthor{Sobhan Hemati}{yyy}
\icmlauthor{Leo Maxime Brunswic}{yyy}
\icmlauthor{Xu Li}{x}
\icmlauthor{Xi Chen}{yyy}
\icmlauthor{Guojun Zhang}{yyy}
\end{icmlauthorlist}


\icmlcorrespondingauthor{Mahdi Beitollahi}{mahdi.beitollahi@huawei.com}
\icmlaffiliation{yyy}{Huawei Noah's Ark Lab, Montreal, Canada}
\icmlaffiliation{x}{Huawei Technologies Canada Inc., Ottawa, Canada}

\icmlkeywords{Machine Learning, ICML}

\vskip 0.3in
]



\printAffiliationsAndNotice{}  

\begin{abstract}
In \emph{one-shot federated learning} (FL), clients collaboratively train a global model in a \emph{single} round of communication. Existing approaches for one-shot FL enhance communication efficiency at the expense of diminished accuracy.

This paper introduces \textbf{FedPFT} (\textbf{Fed}erated Learning with \textbf{P}arametric \textbf{F}eature \textbf{T}ransfer), a methodology that harnesses the transferability of foundation models to enhance both accuracy and communication efficiency in one-shot FL. The approach involves transferring per-client parametric models (specifically, Gaussian mixtures) of features extracted from foundation models. Subsequently, each parametric model is employed to generate synthetic features for training a classifier head. Experimental results on eight datasets demonstrate that FedPFT enhances the communication-accuracy frontier in both centralized and decentralized FL scenarios, as well as across diverse data-heterogeneity settings such as covariate shift and task shift, with improvements of up to 20.6\%.


Additionally, FedPFT adheres to the data minimization principle of FL, as clients do not send real features. We demonstrate that sending real features is vulnerable to potent reconstruction attacks. Moreover, we show that FedPFT is amenable to formal privacy guarantees via differential privacy, demonstrating favourable privacy-accuracy tradeoffs.

\end{abstract}

\section{Introduction}

Federated learning (FL) is a learning paradigm that facilitates training a global model across clients without sharing their raw data \cite{mcmahan2017communication}. Traditional FL takes place iteratively over multiple rounds, wherein each client trains a local model with its own data and subsequently transmits it to a central server for aggregation. The frequent exchange of models, each containing hundreds of millions of parameters, with a central server imposes an intolerable burden of high communication costs on each client, rendering the scalability of FL systems to millions of clients impractical \cite{konevcny2016federated, beitollahi2023federated}. Therefore, reducing the communication cost is desired to support the scaling goals of FL.
 
\begin{figure}[!t]
\centering 
\includegraphics[width=1\columnwidth]{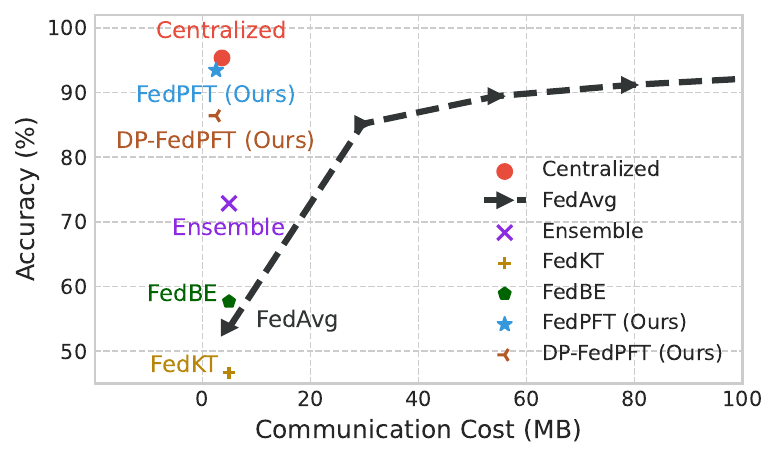}
\vspace{-20pt}
\caption{Comparison of different one-shot FL methods for image classification on Caltech-101 with 50 clients. See Section \ref{sec:comp_sota} for experimental details. FedPFT and DP-FedPFT outperform other one-shot FL methods and are competitive with transmitting real features (Centralized). With more communication budget, multi-round FL (i.e. FedAvg) performs better than one-shot methods.}
\label{fig:intro}
\vspace{-1em}
\end{figure}

\emph{One-shot FL} \cite{guha2019one} is an approach to address the high communication cost of FL by requiring a single communication round. Moreover, one-shot FL addresses several drawbacks of multi-round FL: a) coordinating a multi-iteration FL process across a large number of clients is susceptible to failures, stemming from issues such as client dropout, resource heterogeneity, and real-world implementation challenges; b) support of scenarios where multi-round FL is impractical, e.g., dynamic environments \cite{zhou2022open} where the global model is required to adapt to evolving environments; c) frequent communication poses a higher chance of being intercepted by outsider attacks such as man-in-the-middle attacks \cite{bouacida2021vulnerabilities}.

\begin{figure*}[t!]
\centering
\includegraphics[width=1\textwidth]{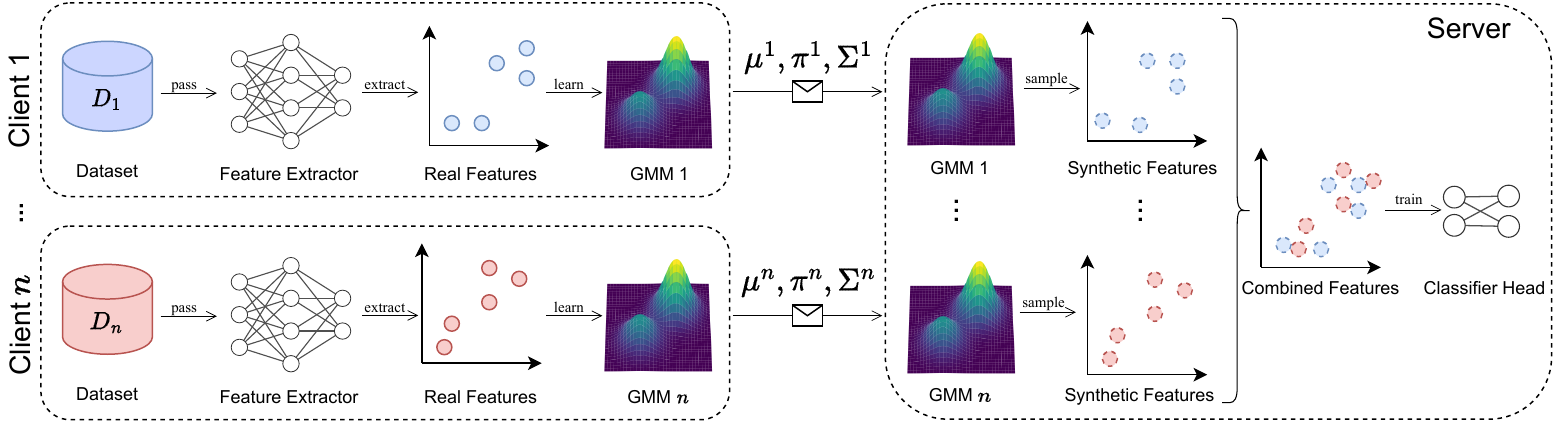} 
\vspace{-2em}
\caption{Illustration of FedPFT in centralized FL. Each client fits GMMs to the distributions of extracted features for each class. Then, GMM's parameters are transmitted to the server, which then samples from these distributions to train a classifier head as the global model.}
\vspace{-0.5em}
\label{fig:gmm_cent}
\end{figure*}

Despite the promising benefits of one-shot FL, existing methods in this category fall short in accuracy compared to their multi-round FL counterparts, primarily due to the strict requirement of one-round knowledge transfer between clients (see Figure \ref{fig:intro}).

Recent strides in foundation models \cite{Bommasani2021FoundationModels} present novel opportunities for efficient knowledge sharing among clients in FL \cite{zhuang2023foundation}. Foundation models such as CLIP \cite{radford2021learning} and GPT series \cite{radford2019language, brown2020language} demonstrate remarkable performance via task-agnostic representations. Features from these models can be leveraged ``as they are'' without extensive fine-tuning, via probing or clustering, to achieve performance on downstream tasks surpassing that of complex, task-specific models. Further, in well-trained foundation models, there exist linear paths in representation space that vary according to semantic axes \cite{mikolov13,harkonen2020}. A corollary is that fitting simple, smooth, parametric distributions in representation space can lead to realistic samples sharing semantic characteristics; on the other hand, this is certainly not true in input space (i.e., Gaussians cannot model the distribution of natural images well enough to generate realistic samples).


\paragraph{Our contributions.} In this paper, we leverage off-the-shelf foundation models to improve both the accuracy and communication efficiency of one-shot FL. We propose \textbf{FedPFT} (\textbf{Fed}erated Learning with \textbf{P}arametric \textbf{F}eature \textbf{T}ransfer). In this data-free knowledge transfer framework, each client learns a parametric model (i.e. Gaussian mixture models (GMMs)) of class-conditional features extracted for each class from a foundation model. Subsequently, clients transmit these parametric models to the server, where they are utilized to generate synthetic features to train a global classifier head as illustrated in Figure \ref{fig:gmm_cent}. FedPFT achieves performance close to centralized training while having client data stay locally on each client, even in scenarios with highly non-iid data distributions, where clients may lack overlapping labels, tasks, or domains. Notably, FedPFT accommodates both heterogeneous communication resources and decentralized network topologies, distinguishing it from prior one-shot federated learning methods. Finally, we demonstrate that FedPFT is amenable to differential privacy (DP), offering formal privacy guarantees with a manageable tradeoff in accuracy. 

Our main contributions include:

\begin{itemize}
    \item We introduce FedPFT, a one-shot FL framework that leverages foundation models to enable parametric feature sharing and significantly enhances the communication-accuracy frontier.

    \item We conduct an extensive evaluation of FedPFT, showcasing its cross-client knowledge transfer capabilities across eight datasets, encompassing various distribution shifts and network topologies. In all tested settings, FedPFT achieves accuracy within a range of 0.03\% to 4\% of centralized accuracy.

    \item We extend FedPFT to offer DP guarantees, demonstrating favorable privacy-accuracy tradeoffs. Additionally, we conduct reconstruction attacks on various feature-sharing schemes and demonstrate the privacy risks of sending real features.

    \item We prove server-side guarantees on the local accuracy of clients under FedPFT. We also analyze the communication-accuracy tradeoffs in GMMs, and show that allocating parameters to increase the number of mixtures yields superior results compared to more granular covariance estimation.

\end{itemize}

\section{Related work}

\textbf{One-shot FL.} Naive parameter averaging methods such as FedAvg \cite{mcmahan2017communication} do not perform well in one-shot FL settings \cite{guha2019one} (see Figure \ref{fig:intro}). To overcome this limitation, researchers have explored diverse strategies, including knowledge distillation (KD) \cite{hinton2015distilling}, ensemble learning \citet{guha2019one}, and generative models \citet{kasturi2023osgan}.

KD methods commonly rely on a \textit{public dataset} with the same distribution of clients' private data at the server to distill the knowledge of the clients to the global model \cite{li2020practical}. However, the availability of public data is not always feasible in practice. In response, \citet{zhang2022dense} and  \citet{zhou2020distilled} have employed generative models and dataset distillation, respectively, to generate the necessary data for KD. However, these methods are also limited to models with batch normalization and non-complex datasets. 

Ensemble learning methods utilize the ensemble of the client's models as the global model at the server. For instance, \citet{guha2019one} averaged the predictions of clients' models at the server, \citet{li2020practical} used major voting, and \citet{chen2020fedbe} sampled global models and combined them via Bayesian model Ensemble. However, as we demonstrate, ensemble methods exhibit poor performance in cases of extreme data heterogeneity, due to local overfitting.

Generative models are used to create a synthetic dataset at the client side and this synthetic dataset is sent to the server for training the global model. For instance, \citet{kasturi2023osgan} trained Generative Adversarial Networks (GAN) at each client. However, training GAN models can be extremely challenging when clients access only a few samples. In another work, \citet{kasturi2020fusion} learned the distribution of categorical features of clients and sent them to the server for sampling.

Recently, \citet{hasan2023calibrated} proposed a one-shot approach using Bayesian learning, where they interpolate between a mixture and the product of the predictive posteriors by considering merging Gaussians in the predictive space.

\textbf{Learning feature statistics.} The exploration of learning feature statistics extends across various domains, including both FL and non-FL domains, such as zero-shot learning \cite{xian2018feature}, few-shot learning \cite{yang2021free}, and continual learning \cite{janson2022simple}. Specifically, \citet{yang2021free} assumed that each dimension of the feature representation follows a Gaussian distribution and transferred the mean and variance of similar classes for few-shot learning. \citet{dennis2021heterogeneity} transferred the mean for unsupervised federated clustering. 



To the best of our knowledge, our work represents the first exploration of learning the distribution of features extracted from foundation models using parametric models, specifically in the context of one-shot federated learning (FL). Additionally, our study uniquely considers a decentralized setting for one-shot FL.


\section{Preliminaries}
\textbf{Federated learning.} The objective of one-shot FL is to learn a model $w$ from data distributed across $I$ clients, each communicating only once. We represent $D_i$ as the local dataset for client $i \in \{1, ..., I \}$ of example-label pairs $(\mathbf{x},y)$, and we denote $n_i :=|D_i|$ as the number of samples for local data of client $i$. The model $w$ is decomposed into $w:= h \circ f$, where $f$ represents a \textit{feature extractor} mapping input $\mathbf{x}$ to a $d$-dimensional embedding, and $h: \mathbb{R}^{d} \rightarrow \mathbb{R}^{C}$ is the classifier head (i.e., linear layer), with $C$ denoting the number of classes. In FL setups, the goal is to minimize the objective function:
\begin{equation} \label{eq:loss}
L(w):=\sum_{i=1}^{I}n_i\mathbb{E}_{(\mathbf{x},y)\sim D_i}[\ell(w; \mathbf{x}, y)],
\end{equation}
where $\ell$ is the cross-entropy loss function.

\textbf{Gaussian mixture models.} We employ Gaussian mixture models (GMMs) as our chosen parametric model, given their concise parameterization and status as universal approximators of densities \cite{scott2015multivariate}. Our approach relies on learning Gaussian mixtures over feature space $\mathbb R^d$. Let $\mathbb{S}^{+}$ denote the set of all $d \times d$ positive definite matrices. We denote by $\mathcal G(K)$ the family of all Gaussian mixture distributions comprised of $K$ components over $\mathbb R^d$. Each density function $g \in \mathcal G(K)$ is identified by a set of tuples $\{(\pi_k,\boldsymbol{\boldsymbol{\mu}}_k,\boldsymbol{\Sigma}_k)\}_{k=1}^K$, where each mixing weight $\pi_k \geq 0$ with $\sum_{k=1}^K \pi_k = 1$, each mean vector $\boldsymbol{\mu}_k \in \mathbb R^d$, and each covariance matrix $\boldsymbol{\Sigma}_k \in \mathbb S^{+}$, satisfying:
\begin{equation} \label{eq:mix}
    g \vcentcolon=  \sum_{k=1}^{K} \pi_k \cdot \mathcal N(\boldsymbol{\mu}_k, \boldsymbol{\Sigma}_k)
\end{equation}
where $\mathcal N(\boldsymbol{\mu}, \boldsymbol{\Sigma})$ refers to the Gaussian density over $\mathbb R^d$ with mean $\boldsymbol{\mu}$ and covariance $\boldsymbol{\Sigma}$. In addition, we denote $\mathcal G_{\text{diag}} (K)$ to denote Gaussian mixtures comprising of diagonal Gaussians, i.e., with the additional constraint that all $\boldsymbol{\Sigma}_k$ are diagonal.
We also denote $\mathcal G_{\text{spher}}$ to Gaussians mixtures with spherical covariances, i.e., each $\boldsymbol{\Sigma_k} \in \{ \lambda \mathbb{I}_d: \lambda \in \mathbb{R}_{\ge 0}\}$. We may also refer to the family of full covariance $\mathcal G(K)$ as $\mathcal G_{\text{full}}(K)$, and use $\mathcal G_{\text{cov}}(K)$ to denote different family types.

\section{Methods}
In this section, We first describe FedPFT for the centralized FL setting, assuming the presence of a centralized server connected to all clients. We also describe the modifications required to adapt FedPFT for (a) the decentralized FL setting, where no single server is connected to all clients; and (b) differential privacy requirements.

\subsection{Centralized FedPFT}
In this conventional FL setup, a central server can aggregate the knowledge from all clients. In the centralized FedPFT scenario, as illustrated in Figure \ref{fig:gmm_cent}, each client $i$ extracts class-conditional features from its local dataset for each available class $c \in \{1, ..., C\}$:
\begin{equation}
F^{i,c} \vcentcolon = \{f(\mathbf{x}); (\mathbf{x},y) \in D_i, y=c\},
\end{equation}
using the pre-trained foundation model feature extractor $f$. Next, each client $i$ learns runs the Expectation Maximization (EM) algorithm \cite{dempster1977maximum} on $F^{i,c}$ to learn a GMM $g^{i,c} \in \mathcal{G}_\text{cov}(K)$ for each class $c$ that approximates $F^{i,c}$. Finally, each client sends its $g^{i,c}$ parameters $\{(\pi_k^{i,c},\boldsymbol{\mu}_k^{i,c},\boldsymbol{\Sigma}_k^{i,c})\}_{k=1}^K$ to the server. On the server side, the server samples class-conditional synthetic features $\Tilde{F}^{i,c}$ from each received $g^{i,c}$ parameters, i.e,
\begin{equation}
    \Tilde{F}^{i,c} \sim g^{i,c} = \sum_{k=1}^{K} \pi_k^{i,c} \cdot \mathcal N(\boldsymbol{\mu}_k^{i,c}, \boldsymbol{\Sigma}_k^{i,c}) 
\end{equation}
with the size of $|F^{i,c}|$. Then, the server combines class-conditional synthetic features $\Tilde{F}^{i,c}$ from all the clients and classes to create synthetic feature dataset $\Tilde F$ as follows,
\begin{equation}
\Tilde D = \bigcup_{i=1}^I \bigcup_{c =1}^C \{(\mathbf{v},c): \mathbf{v} \in \Tilde F^{i,c}\}.
\end{equation}
Finally, the server trains a classifier head $h$ on $\Tilde F$, minimizing $\mathbb E_{(\mathbf{v},y)\sim \Tilde D}[\ell(h;\mathbf{v},y)]$ where $\ell$ is the cross-entropy loss. The trained, global classifier head $h$ is then sent back to the clients, and clients can use $w = h \circ f$ as the global model. This process is described in Algorithm \ref{alg:algorithm}. 


\begin{figure}[t!]
\centering
\includegraphics[width=1\columnwidth]{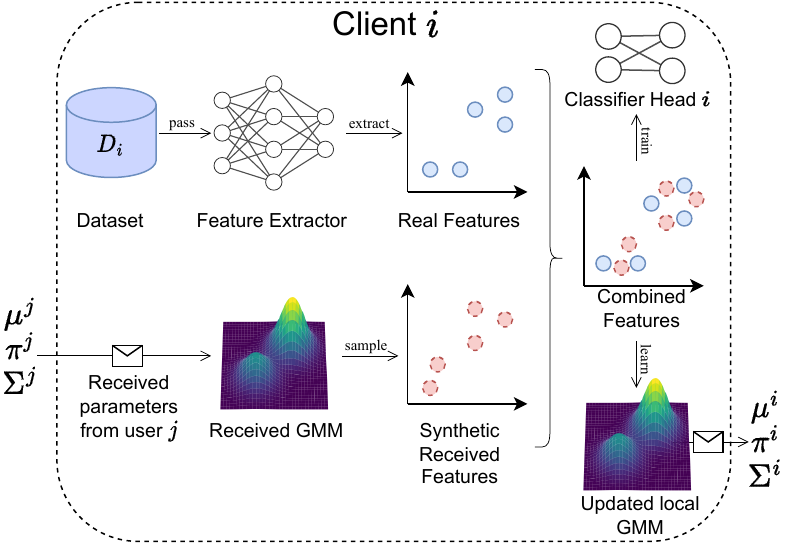} 
\vspace{-2em}
\caption{Illustration of the FedPFT framework for decentralized FL. Each client updates the received statistics of GMMs with its local data and transfers it to other clients.
}
\vspace{-0.5em}
\label{fig:gmm_dist}
\end{figure}

The described FedPFT method operates in a one-shot manner, meaning that each client needs to communicate with the server only once. Further, FedPFT's aggregation method in the server is data-free meaning it does not require any additional data to aggregate clients' features. Finally, since the server has access to class-conditional features for each class, it can create classifiers using only a subset of classes.

\begin{algorithm}[tb]
\caption{FedPFT for centralized, one-shot FL.}
\label{alg:algorithm}
\begin{algorithmic}[1] 
\STATE \textbf{Input}: Client datasets $D_1,...,D_I$, pre-trained feature extractor $f$. \\
\STATE \textbf{Parameters}: Number of clients $I$, number of classes $C$, number of mixtures $K$,  covariance type \text{cov}.\\
\STATE  \textbf{Output}: Model $w := h \circ f$
\STATE \texttt{// Client side:} 
\FOR{each client $i \in \{1, ..., I \}$}
\FOR{each class $c \in \{1, ..., C\}$}
\STATE Let $F^{i,c} \vcentcolon= \{f(\mathbf{x}): (\mathbf{x},y) \in D_i, y=c\}$.
\STATE Run the EM algorithm on $F^{i,c}$ to learn a GMM  $g^{i,c} \in \mathcal{G}_\text{cov}(K)$.
\STATE Send $g^{i,c}$ parameters $\{(\pi_k^{i,c},\boldsymbol{\mu}_k^{i,c},\boldsymbol{\Sigma}_k^{i,c})\}_{k=1}^K$ to the server.
\ENDFOR
\ENDFOR
\STATE  \texttt{// Server side:}
\FOR{each received $\{(\pi_k^{i,c},\boldsymbol{\mu}_k^{i,c},\boldsymbol{\Sigma}_k^{i,c})\}_{k=1}^K$ set}
\STATE Sample synthetic features $\Tilde{F}^{i,c} \sim g^{i,c} = \sum_{k=1}^{K} \pi_k^{i,c} \cdot \mathcal N(\boldsymbol{\mu}_k^{i,c}, \boldsymbol{\Sigma}_k^{i,c})$ of size $|F^{i,c}|$.
\ENDFOR
\STATE Let $\Tilde D = \bigcup_{i=1}^I \bigcup_{c =1}^C \{(\mathbf{v},c): \mathbf{v} \in \Tilde F^{i,c}\}$.
\STATE Train a classifier head $h$ on $\Tilde F$, minimizing $\mathbb E_{(\mathbf{v},y)\sim \Tilde D}[\ell(h;\mathbf{v},y)]$ where $\ell$ is the cross-entropy loss.
\STATE \textbf{return} model $w = h \circ f$.
\end{algorithmic}
\end{algorithm}

\subsection{Decentralized FedPFT} \label{sec:decen}

\begin{figure*}[ht]
    \centering
    \begin{minipage}{0.49\textwidth}
        \centering
        \includegraphics[width=0.99\textwidth]{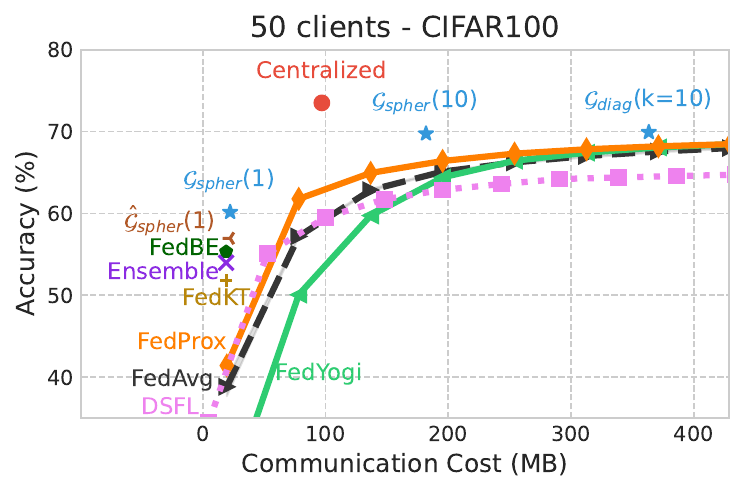} 
    \end{minipage}\hfill
    \begin{minipage}{0.49\textwidth}
        \centering
        \includegraphics[width=0.99\textwidth]{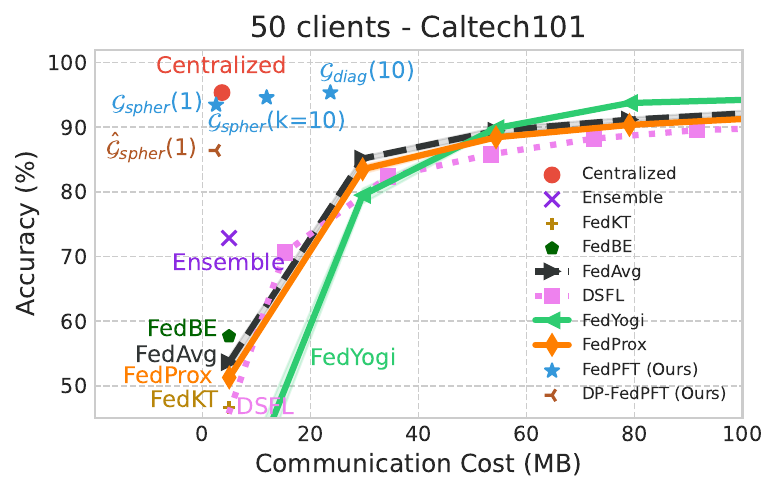} 
    \end{minipage}
    \vspace{-1em}
    \caption{FedPFT vs existing one-shot and multi-round FL methods in Centralized setting with CIFAR100 (left) and Caltech 101 (right) dataset. FedPFT ($\mathcal{G}$) and DP-FedPFT ($\hat{\mathcal{G}}$) surpass other one-shot FL methods, and are competitive with sending raw features (Centralized).}
    \vspace{-1em}
    \label{fig:com_acc}
\end{figure*}

In decentralized FL, there is no centralized server, and clients are all connected in an ad-hoc manner. Therefore, each client has the responsibility of aggregating the knowledge of its dataset with other clients and passing the knowledge to the next client.

FedPFT for decentralized FL is illustrated in Figure \ref{fig:gmm_dist}. In this method, similar to centralized FL, each client $i$ creates class conditional features from its local dataset for each label $c$, i.e., $F^{i,c}$. Then, client $i$ samples from the received GMMs from client $j$ to generate synthetic class-conditional features $\Tilde{F}^{j,c}$. Next, client $i$ learns runs the Expectation Maximization (EM) algorithm on $F^{i,c} \cup \Tilde{F}^{j,c}$ to learn a GMM $g^{i,c} \in \mathcal{G}_\text{cov}(K)$ for each class $c$ that approximates the union of $F^{i,c}$ and $\Tilde{F}^{j,c}$. The parameters of $g^{i,c}$ are sent to the next client. At the same time, each client can use the combined features $F^{i,c} \cup \Tilde{F}^{j,c}$ to train its local classifier head $h_i$. By passing GMMs between clients, knowledge of each client is accumulated and propagated between clients with just one round of communication where the last client has the knowledge of all the clients.

\subsection{DP-FedPFT}
FedPFT's goal is to transfer the parameters of GMMs without leaking clients' private information in their dataset. For formal privacy guarantees, we employ differential privacy \citep[DP,][]{dwork2006}. 
To make FedPFT differentially private, we use the Gaussian mechanism \cite{dwork2006,dwork2014algorithmic} to privatize the release of all mean vectors and covariance matrices. The following theorem provides the $(\epsilon, \delta)$-differential privacy guarantee for FedPFT in the case of Gaussians ($\mathcal{G}_{\text{full}}(K=1)$). 

\begin{theorem} \label{th:dp-m=1}
Suppose the feature embedding $f$ satisfies $\|f\|_2 \leq 1$. Let $\hat{\boldsymbol\mu}(\cdot)$ and $\hat{\mathbf{\Sigma}}(\cdot)$ be the estimator of mean and covariance, respectively.
Define the Gaussian mechanism 
\begin{align}
        &\mathcal M : D \mapsto (\widetilde{\boldsymbol{\mu}}(D),
\widetilde {\boldsymbol{\Sigma}}(D)), \\
        &\widetilde{\boldsymbol\mu}(D) = 
        \hat{\boldsymbol\mu}(f(D)) + \Delta \boldsymbol\mu,\\
    &\widetilde{\boldsymbol\Sigma}(D) = {\rm Proj}_{\mathbb{S}_+}( \hat{\mathbf{\Sigma}}(f(D))+ \Delta\boldsymbol 
    \Sigma), 
\end{align}
where the elements of vector $\Delta\boldsymbol{\mu}$ and matrix $\Delta \boldsymbol \Sigma$  are sampled from independent $\mathcal{N}\left(0,\left(\frac{4}{n_i\epsilon}\sqrt{5 \ln(4/\delta)} \right) ^2\right)$, and ${\rm Proj}_{\mathbb{S}_+}$ is the projection onto the set of positive semi-definite matrices. 
Then, the Gaussian mechanism $\mathcal M$ satisfies $(\epsilon, \delta)$-differential privacy.
\end{theorem}
See Appendix \ref{apx:proofs} for the proof.



\section{Experiments}

We examine the knowledge transfer capabilities of FedPFT and DP-FedPFT in the one-shot setting, comparing them with state-of-the-art methods across various data heterogeneity settings and network topologies. Our experiments support the claim that FedPFT: (1) compares favorably against existing one-shot FL methods; (2) succeeds in a variety of extreme client distribution scenarios challenging for FL; and (3) supports decentralized network topologies.

\begin{table}[h]
  \caption{Summary of datasets and foundation models}
  \label{tbl:datasests}
    \centering
     \resizebox{\columnwidth}{!}{\begin{tabular}{lllll}
    \toprule
    Dataset   & Min  size      & \# Train              & \# Classes & Foundation Model  \\
    \midrule
    CIFAR10   &  (32, 32)       & 50,000                    & 10        & ResNet-50                     \\  
    CIFAR100   &  (32, 32)       & 50,000                   & 100        & ResNet-50                     \\  
    PACS (P) &  (224, 224)       & 1,336                   & 7        & ViT-B/16                   \\  
    PACS (S) &  (224, 224)       & 3144                   & 7        & ViT-B/16                     \\
    Office Home (C) &  (18, 18)       & 3492                   & 65        & ViT-B/16                    \\
    Office Home (P) &  (75, 63)       & 3551                   & 65        & ViT-B/16                    \\
    Caltech101 &  (200, 200)       & 6084                   & 101        & CLIP, ViT-B/32                   \\
    Stanford Cars &  (360, 240)       & 12948                   & 196        & CLIP, ViT-B/32                    \\
    Oxford Pets &  (108, 114)       & 3680                     & 37        & CLIP, ViT-B/32                    \\
    Food101 &  (512, 512)       & 75750                   & 101        & CLIP, ViT-B/32                    \\
    \bottomrule
  \end{tabular}}
  \vspace{-1.5 em}
\end{table}

\subsection{Experimental setting}
 \paragraph{Datasets and foundation models.} We use 8 vision datasets including CIFAR10/100 \cite{krizhevsky2014cifar}, PACS \cite{li2017deeper}, Office Home \cite{venkateswara2017deep}, Caltech101 \cite{li_andreeto_ranzato_perona_2022}, Stanford Cars \cite{krause20133d}, Oxford Pets \cite{parkhi2012cats}, and Food101 \cite{bossard2014food} and three foundation models including ResNet-50 \cite{he2016deep} pre-trained on ImageNet, ViT-B/16 \cite{dosovitskiy2020image} pre-trained on ImageNet, and the CLIP \cite{radford2021learning} image encoder as our feature extractor $f$. Table \ref{tbl:datasests} provides a summary of the datasets and the corresponding feature extractors used.

\begin{table*}[!th]
  \caption{FedPFT in three extreme shifts in two-client decentralized FL. We format \textbf{first} and{\transparent{0.6} oracle} results.}
  \label{tbl:extreme-shift}
    \resizebox{\textwidth}{!}{\begin{tabular}{l cccc cccc cccc}
    \toprule
    & \multicolumn{4}{c}{Disjoint Label shift} & \multicolumn{4}{c}{Covariate shift}  & \multicolumn{4}{c}{Task shift} \\
     \cmidrule(lr){2-5}  \cmidrule(lr){6-9}   \cmidrule(lr){10-13} 
    & \multicolumn{2}{c}{CIFAR-10}  & \multicolumn{2}{c}{CIFAR-100}  & \multicolumn{2}{c}{PACS (P$\rightarrow$S)}  & \multicolumn{2}{c}{Office Home (C$\rightarrow$P)}  & \multicolumn{2}{c}{Birds $\rightarrow$ Cars}  & \multicolumn{2}{c}{Pets  $\rightarrow$ Food} \\
    \cmidrule(lr){2-3}  \cmidrule(lr){4-5} \cmidrule(lr){6-7} \cmidrule(lr){8-9} \cmidrule(lr){10-11}   \cmidrule(lr){12-13} 
    Methods   & Accuracy      & Comm.       & Accuracy      & Comm.    & Accuracy      & Comm.       & Accuracy      & Comm.    & Accuracy      & Comm.       & Accuracy      & Comm.  \\
    \midrule
    \tr Centralized  & \tr 90.85  \tiny{ $\pm$ 0.03}  & \tr 97 MB   & \tr 73.97  \tiny{ $\pm$ 0.06}   &\transparent{0.4} 97 MB    &  \tr 89.15  \tiny{ $\pm$ 0.17}              & \tr 2.5 MB            & \tr 82.00  \tiny{ $\pm$ 0.16}              & \tr 6.3 MB    & \tr 81.88  \tiny{ $\pm$ 0.06}              & \tr 6.7 MB            & \tr 88.48  \tiny{ $\pm$ 0.05}              & \tr 3.6 MB         \\
    Ensemble         & 80.18  \tiny{ $\pm$ 0.30}      & 80 KB       & 57.94  \tiny{ $\pm$ 0.22}      & 0.7 MB       & 79.59  \tiny{ $\pm$ 0.83}              &  10.5 KB            & 71.36  \tiny{ $\pm$ 0.46}              &  96 KB    & 58.33  \tiny{ $\pm$ 2.01}              & 0.3 MB            & 83.54  \tiny{ $\pm$ 0.21}              & 0.2 MB   \\
    Average         & 77.66  \tiny{ $\pm$ 1.04}       & 80 KB       & 56.82  \tiny{ $\pm$ 0.22}      & 0.7 MB    & 77.83  \tiny{ $\pm$ 0.23}              & 10.5 KB            & 69.69  \tiny{ $\pm$ 0.69}              &  96 KB    & 72.65  \tiny{ $\pm$ 0.13}              & 0.3 MB            & 83.25  \tiny{ $\pm$ 0.59}              & 0.2 MB       \\    
    KD              & 74.22  \tiny{ $\pm$ 0.42}       & 80 KB       & 55.62  \tiny{ $\pm$ 0.20}      & 0.7 MB       & 67.69  \tiny{ $\pm$ 2.47}             & 10.5 KB            & 72.15  \tiny{ $\pm$ 0.75}              & 96 KB   & 40.46  \tiny{ $\pm$ 0.31}              & 0.3 MB            & 43.67  \tiny{ $\pm$ 0.04}              & 0.2 MB       \\     
    \midrule
    $\mathcal{G}_{\text{diag}}$(K=10)    &  {86.19}  \tiny{ $\pm$ 0.15}      & 0.4 MB       & {69.97}  \tiny{ $\pm$ 0.07}            & 3.9 MB     & \bf{89.12}  \tiny{ $\pm$ 0.18}             &  0.2 MB          & {80.56}  \tiny{ $\pm$ 0.31}              & 1.8 MB  & {81.74}  \tiny{ $\pm$ 0.06}             & 1.9 MB          & {88.22}  \tiny{ $\pm$ 0.07}              & 0.7 MB     \\ 
    $\mathcal{G}_{\text{diag}}$(K=20)    & \bf{86.89}  \tiny{ $\pm$ 0.02} & 0.8 MB       & \bf{70.31}  \tiny{ $\pm$ 0.10}        & 7.8 MB    & {89.00}  \tiny{ $\pm$ 0.15}             & 0.4 MB          & \bf{80.94}  \tiny{ $\pm$ 0.16}              & 3.6 MB    & \bf{81.75}  \tiny{ $\pm$ 0.04} & 3.8 MB & \bf{88.25}  \tiny{ $\pm$ 0.08} & 1.4 MB \\ 
    \bottomrule
  \end{tabular}}
  \vspace{-1 em}
\end{table*}

 \paragraph{Implementation.} We provide the code for implementing GMMs in FedPFT in Appendix \ref{code:gmm}. For transmitting GMM parameters, we use a 16-bit encoding. We use $\mathcal{G}_{\text{cov}}(K)$ and $\hat{\mathcal{G}}_{\text{cov}}(K)$ to denote FedPFT and DP-FedPFT, respectively, where \textit{cov} represented covariance types. We set $\epsilon=1$, $\delta = 1/|D^{i,c}|$, and use $K=1$ mixture components for all DP-FedPFT experiments. We run all the experiments for three seeds and report the mean and standard deviation of the accuracy on the hold-out test dataset, along with its communication cost. Further experimental details can be found in Appendix \ref{sec:exp-details}.

\paragraph{Baselines.}
We present the results of centralized training with raw features (\textit{Centralized}) as the oracle result. We also include the results of an \textit{Ensemble} comprising locally trained classifier heads from clients, where the class with the highest probability across the models is selected.

\subsection{Comparing to existing one-shot and multi-round FL methods} \label{sec:comp_sota}
In Figure \ref{fig:com_acc}, we compare the performance of FedPFT with state-of-the-art one-shot FL methods in centralized settings with 50 clients for training the classifier head. We conduct tests on both CIFAR100 with the ResNet-50 feature extractor and the Caltech101 dataset with the CLIP ViT feature extractor, where the samples are distributed across clients according to Dirichlet ($\beta=0.1$). The benchmarked multi-round methods include FedAvg \cite{mcmahan2017communication}, FedProx \cite{li2020federated}, DSFL \cite{beitollahi2022dsfl}, and FedYogi \cite{reddi2020adaptive}, along with single-shot methods such as FedBE \cite{chen2020fedbe}, FedKT \cite{li2020practical}, and Ensemble.

Figure \ref{fig:com_acc} shows that FedPFT and DP-FedPFT beat other one-shot FL methods; and are competitive with sending raw features (Centralized). With more communication budget, multi-round FL methods perform better than existing one-shot methods. Also, Figure \ref{fig:com_acc} shows different tradeoffs for different numbers of mixtures $K$ and covariance types.  Additionally, Figure \ref{fig:com_acc} demonstrates different tradeoffs for varying numbers of mixtures $K$ and covariance types. See Section \ref{sec:features_dist} for more discussion on these tradeoffs. Further details of this experiment can be found in the Appendix. 

\begin{figure}[h]
\centering 

\includegraphics[width=\columnwidth]{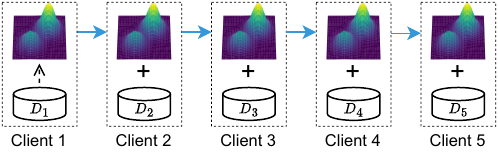} %
\vspace{-2em}
\caption{Five clients in a linear topology. Each client updates its received GMM with its local data and sends it to the next client.}
\label{fig:5clients}

\includegraphics[width=\columnwidth]{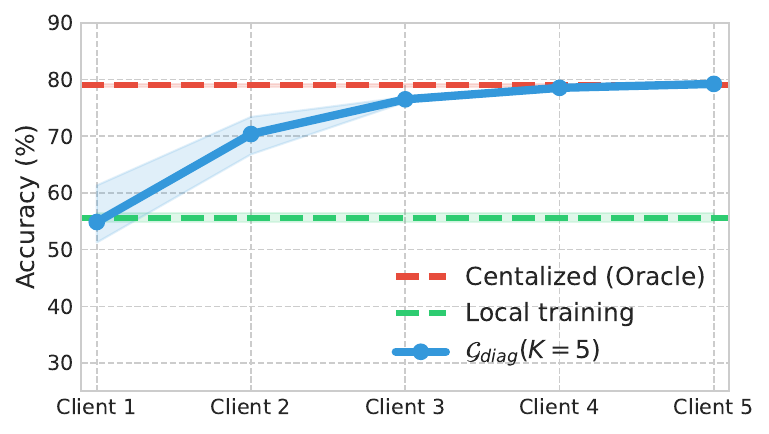} 
\vspace{-2em}
\caption{Results of FedPFT with 5 clients in linear topology as illustrated in Figure \ref{fig:5clients}.}
\vspace{-1em}
\label{fig:5clients_results}
\end{figure}

\begin{figure*}[t!]
\centering
\includegraphics[width=1\textwidth]{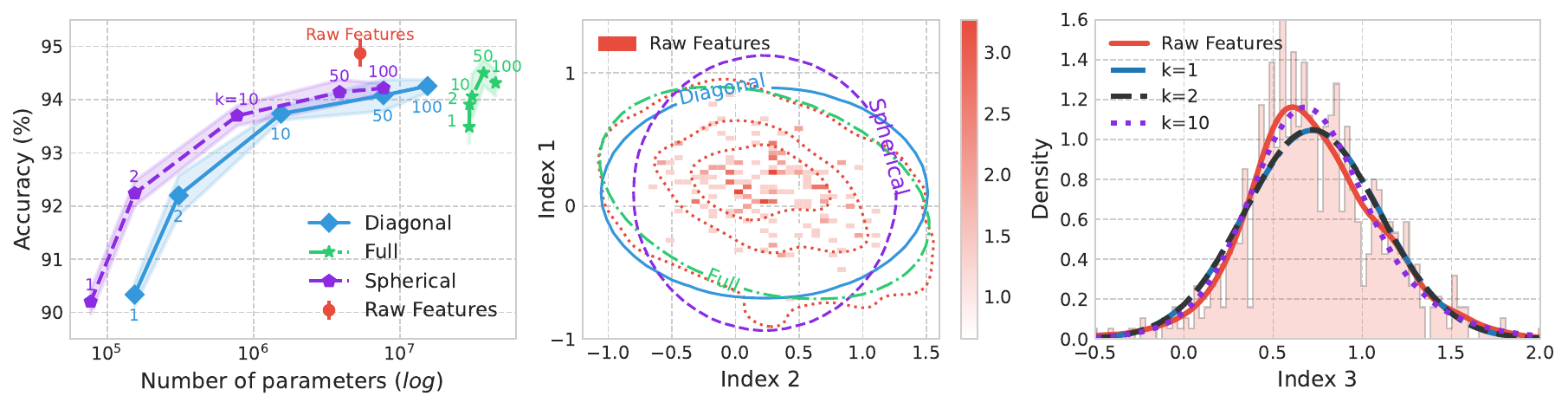} 
\vspace{-2em}
\caption{Comparing real and synthetic feature distributions using the Caltech101 dataset. (Left) Classifier accuracy on raw vs. synthetic features from various GMMs. (Middle) 2-dimensional and (Right) 1-dimensional distribution of random indexes of real features vs GMMs' counterparts with different covariance types and number of mixtures.}
\vspace{-1em}
\label{fig:features}
\end{figure*}

\subsection{Extreme shifts in two-client decentralized FL} \label{sec:extreme-shift}
By design, FedPFT is agnostic to data heterogeneity. To assess this claim, we examine decentralized FL settings with two clients: source and destination. These clients exhibit significantly different training distributions. Specifically, we explore three types of extreme shift scenarios—label shift, covariate shift, and task shift—between clients, where the source can communicate only once to the destination client. We report the performance of the destination's trained classifier head on both clients' test datasets in Table \ref{tbl:extreme-shift}.

For baselines, we compare to KD, where: (1) each client locally trains a classifier head; (2) the source client sends its local classifier head to the destination client; and (3) the destination client distills the received classifier head to its local classifier head. We report ensembling and averaging locally trained classifier heads as baselines.

\paragraph{Disjoint label shift.}
In our disjoint label shift setups, our source client only has samples from the first half of labels (0-4 for CIFAR-10 or 0-49 for CIFAR-100), and the destination client has the other half.

\paragraph{Covariate shift.}
In our covariate shift setups, the two clients have access to the two most distinctive domains of PACS and Office Home datasets according to \cite{hemati2023cross}. Specifically, for PACS, we consider the scenario where the source has access only to Photo (P) images, while the destination has access to Sketch (S) images. For Office Home, the source has access to Clipart (C) images, and the destination has access to Product (P) images.

\paragraph{Task shift.}
In our task shift setups, the two clients have access to two different datasets with distinct tasks. First, we consider the scenario where the source has bird images from Caltech101, while the destination has car images from Stanford Cars. In the second experiment, the source has access to pet images from Oxford Pets, and the destination has food images from Food101.

\paragraph{Results.} Table \ref{tbl:extreme-shift} demonstrates that FedPFT succeeds in various extreme client distribution scenarios and highlights the limitations of vanilla averaging, ensembling, and KD methods.

\subsection{Linear topology}
In this setup, we demonstrate the propagation and accumulation of knowledge using FedPFT in decentralized FL with five clients in a linear topology where each client has access to 100 i.i.d. samples of the CIFAR10 dataset, as illustrated in Figure \ref{fig:5clients}. Using GMMs, we transfer knowledge of client 1 to client 5 in four communication rounds and report the performance of each client's classifier head trained on its received GMM on the entire 5 client dataset in Figure \ref{fig:5clients_results}. We also compare FedPFT with centralized training and local training, where we train a classifier head for each client's dataset. Figure \ref{fig:5clients_results} shows that as GMMs pass through clients, they can accumulate and propagate the knowledge and achieve performance close to 1.8\% of centralized training.


\section{Analysis}
In this section, we provide a comprehensive analysis of FedPFT, examining its accuracy, communication cost, and privacy characteristics. In particular, we: (1) examine the tradeoff between accuracy and communication cost for different families of GMMs; (2) prove theoretical guarantees on local client accuracy; (3) estimate communication costs; and (4) analyze privacy leakage in FedPFT, compared to sending raw features.

\subsection{How well do GMMs model feature distributions?} \label{sec:features_dist}
FedPFT relies on learning the distribution of features of each class using GMMs. We aim to assess how effectively GMMs can estimate class-conditional features. We measure the accuracy gap between two classifier heads: one trained on real features and another on synthetic features generated using GMMs. Monitoring the accuracy gap helps us evaluate the discriminative power of synthetic features. We examine the effect of the number of mixtures $K$ and the type of covariance matrix. For this experiment, pictured in Figure \ref{fig:features}, we use CLIP ViT-B/32 features on the Caltech101 dataset. For various families of GMMs, we plot accuracy and the total number of statistical parameters.


Figure \ref{fig:features} (left) shows that 10-50 Gaussians are sufficient to represent raw extracted features with less than a 1\% drop in accuracy. Notably, GMMs with a spherical covariance matrix exhibit better tradeoffs between communication and accuracy compared to full/diagonal covariance matrices. Additionally, we illustrate the 1-dimensional (Figure \ref{fig:features} right) and 2-dimensional (Figure \ref{fig:features} middle) density of random indexes of raw features and compare them with GMMs.


\subsection{FedPFT guarantees local accuracy of each client}
In this section, we assume each client's local feature dataset is dequantized (say by adding a noise).
\begin{theorem}
    For any classifier head $h$, with $\ell_i^{0-1}$ and $\widetilde{\ell_{i,c}^{0-1}}$ the  0-1 losses of $h$ on $f(D_i)$ and $\widetilde F^{i,c}$ respectively,  $\mathcal H^{i,c}$  the self-entropy of the distribution $F^{i,c}$, and $\mathcal L_{EM}^{i,c}$ the class-wise log-likelihood of the EM algorithm:
    \begin{equation}\label{eq:local_guarantee}
        \ell_i^{0-1} \leq \mathbb E_c\left[ 2\widetilde{\ell_{i,c}^{0-1}} - (\widetilde{\ell_{i,c}^{0-1}}) ^2 + \frac{1-\widetilde{\ell_{i,c}^{0-1}}}{\sqrt{2}}\sqrt{\mathcal H^{i,c}-\mathcal L_{EM}^{i,c}}  \right].\nonumber
    \end{equation}
    
\end{theorem}

This theorem shows that the 0-1 loss $\ell_{i}^{0-1}$ of the server-trained classifier head $h$ on the local client data is bounded by a term that depends on: a) its class-wise 0-1 loss on synthetic features $\widetilde{\ell_{i,c}^{0-1}}$; b) the loss of the local EM algorithm $\mathcal L_{EM}^{i,c}$ on each class, and c) local class-wise self-entropy $\mathcal H^{i,c}$. 

This bound enables FedPFT to offer a server-side guarantee that the global classifier head's local accuracy for each client is high \emph{if its accuracy on synthetic features is high}, contingent upon an estimation of the EM loss and self-entropy provided by the clients.
Note the 0-1 loss on the server side on synthetic data should be half the client-side target to allow for an achievable EM loss target.

\subsection{Communication cost of FedPFT} 
Here, we estimate the communication cost of FedPFT and compare it to sending the classifier head or sending the raw features. Denoting by $d$ the feature dimension, $K$ the number of components, and $C$ the number of classes, we calculate the communication cost of FedPFT for different covariance families:
{\setlength{\abovedisplayskip}{10pt}
\setlength{\belowdisplayskip}{10pt}
\setlength{\abovedisplayshortskip}{20pt}
\setlength{\belowdisplayshortskip}{20pt}
\begin{align}
    &\text{Cost}(\mathcal{G}_\text{full}): (2d+\frac{d^2-d}{2})+1)KC  \sim \mathcal{O}(d^2CK),  \label{eq:com_cost}\\ 
    &\text{Cost}(\mathcal{G}_\text{diag}): (2d + 1)KC  \sim \mathcal{O}(2dCK), \\ 
    &\text{Cost}(\mathcal{G}_\text{spher}): (d + 2)KC  \sim \mathcal{O}(dCK). \label{eq:com_cost_spher}
\end{align}}Equations (\ref{eq:com_cost} - \ref{eq:com_cost_spher}) indicate that the communication cost of FedPFT is independent of the number of samples of each client $n_i$. Therefore, FedPFT can scale better compared to sending raw data or raw features when the number of samples is large. More specifically, when $n_i \gtrsim
 2dCK$, it is more communication efficient to send $\mathcal{G}_\text{diag}(K)$ than send the raw features. This is also shown in Figure $\ref{fig:features}$ (left). Similarly, equation (\ref{eq:com_cost_spher}) shows that the communication cost of $\mathcal{G}_\text{spher}(K=1)$ is equal to the communication cost of sending the classifier head, which is $(Cd + C)$. Therefore, GMMs can have the same communication cost as FedAvg. Further, FedPFT supports heterogeneous communication resources, as each client can utilize a different $K$.


\subsection{Evaluating against reconstruction attacks}
We conduct reconstruction attacks on the various feature-sharing schemes described in this paper using the CIFAR10 dataset. Figure \ref{fig:privacy} verifies the vulnerability of raw feature sharing -- an attacker with access to in-distribution data (i.e. CIFAR-10 train set) can obtain high-fidelity reconstructions of the private data of clients (i.e. CIFAR-10 test set). This necessitates sharing schemes beyond sending raw features. 

The attack we present involves training a generative model on (extracted feature, image) pairs and then performing inference on received feature embeddings. We apply the same attack on features sampled via FedPFT and DP-FedPFT. The reconstructions are \emph{set-level}, and we present the closest image among the entire reconstruction set by SSIM. The resulting reconstructions do not resemble the original image (Figure \ref{fig:privacy}). Quantitative results reporting image similarity metrics can be found in Table \ref{tab:oracle-recon-compare}. 

For full experimental details, further quantitative results, and samples, please see Appendix \ref{sec:recon-details}. We also include a full description of our threat model and results on different backbones (e.g. our attacks are stronger on MAEs), which could not be included in the main body due to space limitations.

\begin{figure}[t]
\centering
\includegraphics[width=0.99\columnwidth]{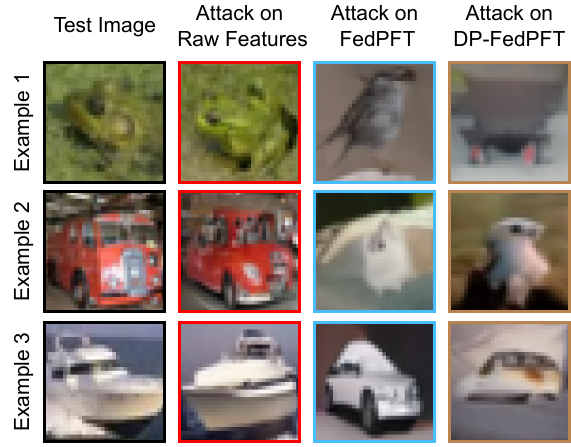}
\caption{Results of reconstruction attacks on feature-sharing schemes using three random test images from CIFAR-10. Attackers can reconstruct raw features (middle-left) to generate images resembling real data (left). However, the same reconstruction on FedPFT (middle-right) and DP-FedPFT (right) does not resemble the real image. \emph{For more reconstruction examples see Figures \ref{fig:raw-feature-recon} and \ref{fig:gmm-recon} in the Appendix. \label{fig:privacy}}}
\end{figure}

\begin{table}[!t]
    \fontsize{9pt}{9pt}\selectfont
    \centering
    \begin{tabular}{lrrr}
    \toprule
    \multirow{2}{*}{\shortstack{\textbf{Reconstruction} \\ \textbf{source}}} & \multicolumn{3}{c}{\textbf{Similarity measure}} \\
    \cmidrule{2-4}
     & \multirow{1}{*}{\emph{PSNR $\uparrow$}} & \multicolumn{1}{c}{\emph{LPIPS $\downarrow$}} & \multicolumn{1}{c}{\emph{SSIM $\uparrow$}} \\
    \midrule
    Raw features & 16.5 & .257 & .535 \\
    FedPFT & 14.7 & .423 & .508  \\
    DP-FedPFT & 14.6 & .511 & .483 \\
    \midrule
    Train set & 14.7 & .389 & .500  \\
    \bottomrule
    \end{tabular}
    \caption{Quantitative metrics for \emph{set-level reconstruction}. We report image similarity on the top 1\% ($n=90$) of test set images, by their SSIM to a member of the reconstruction set. As a baseline, we report the results of treating the train set as a reconstruction set. For further elaboration on our threat model, see Appendix \ref{sec:recon-details}.}\label{tab:oracle-recon-compare}
    \vspace{-3em}
\end{table}





\vspace{-1em}
\section{Conclusion}
We introduce FedPFT, a one-shot federated learning (FL) method leveraging foundation models for better accuracy and communication efficiency. FedPFT utilizes per-client parametric models of features extracted from foundation models for data-free knowledge transfer. 

Our experiments show that FedPFT improves the communication-accuracy frontier in a wide range of data-heterogeneity settings. Moreover, we showed that FedPFT is amenable to formal privacy guarantees via differential privacy, exhibiting good privacy-accuracy tradeoffs. Our theoretical analysis demonstrates that FedPFT has server-side guarantees on the local accuracy of clients. Additionally, we conduct reconstruction attacks on feature-sharing schemes and demonstrate the privacy risks of sending real features.


\bibliography{main}
\bibliographystyle{icml2023}

\newpage
\appendix
\onecolumn

\section{Code} \label{code:gmm}
In this section, we provide code for the main component of FedPFT which is extracting the GMMs for class-conditional features. We consider the case of $k=1$ and refer readers to the DP-EM \citet{pmlr-v54-park17c} method for the general case. Also, in this section, we assume each client can access only one class for simplicity and without losing generality. 
\begin{lstlisting}[language=Python]

from sklearn.mixture import GaussianMixture as GMM
import numpy as np


def create_gmm(features: np.array, labels: np.array) -> Dict:
    '''
    Creates a diagonal GMM with 10 mixtures for each class in features and
    sample from it to create synthetic dataset
    
    
    Args:
        features(np.array): all the features of training dataset
        labels(np.array): all the associated labels of features
    returns 'synthetic_dataset' 
    
    '''
    synthetic_dataset = {}
    for label in list(set(labels)):
        conditional_features = features[labels==label]
        
        # Create a GMM for 'label'
        gmm = GMM(
            n_components=10,
            covariance_type='diag',
            )
        gmm.fit(conditional_features)
        
        # Sample from the GMM for 'label' at the server
        gmm_feature, _ = gmm.sample(conditional_features.shape[0])

        # Add it to the synth dict
        synthetic_dataset[label] = gmm_feature

    return synthetic_dataset
\end{lstlisting}

\section{DP-FedPFT}
This section provides detailed definitions and proofs for Theorem \ref{th:dp-m=1}. 
\label{apx:proofs}

\begin{definition}
(Differential Privacy, \citealt{dwork2006}). A randomized algorithm $\mathcal{M} : \mathcal{U} \rightarrow \Theta $ is ($\epsilon, \delta)$-
differentially private if for every pair of neighboring datasets $ D, D' \in \mathcal{U}$ for all $S\subseteq \Theta$, we have 
\begin{equation}
\Pr[\mathcal{M}(D) \in S] \leq {\rm e}^{\epsilon} \Pr[\mathcal{M}(D') \in S] + \delta.
\end{equation}
\end{definition}

\begin{lemma} \label{lem:gauss}
    (Gaussian Mechanism, \citealt{dwork2006,dwork2014algorithmic}).
    Let $\epsilon>0$ and  let $g : D^{n} \rightarrow \mathbb{R}^{d}$ be a function with $\ell_2$-sensitivity ${\Delta}_g$. Then the Gaussian mechanism 
    \begin{equation}
        \mathcal{M}(D) := g(D) + \mathcal{N}\left(0, {\left(\frac{{\Delta}_g\sqrt{2ln(2/\delta)}}{\epsilon} \right)}^2 \cdot \mathbb{I}_{d\times d}\right),
        \label{eq:gauss}
    \end{equation}
 satisfies $(\epsilon, \delta)$-differential privacy.
\end{lemma}

\begin{proof}{(Theorem \ref{th:dp-m=1})} 
We wish to apply Lemma \ref{lem:gauss} with $g$ mapping a feature dataset to its Gaussian mixture approximation obtained via EM-method. The post-processing property (Proposition 2.1 of \citealt{dwork2014algorithmic}) ensures that the projection onto the positive semi-definite symmetric matrices preserves differential privacy.
For $k=1$, denote by $\hat{\boldsymbol{\mu}},\hat{\boldsymbol{\Sigma}}$ the mean and covariance matrix of this Gaussian approximation; in this instance they coincide with the mean and covariance matrix of the dataset. 
We may use the $\ell_2$-sensitivity of the function $(\hat{\boldsymbol{\mu}},\hat{\boldsymbol{\Sigma}})$ to apply the aforementioned Lemma.

To begin with, for any triplet of independent random variables $(X,Y,\epsilon)$ with $X,Y$ with values in $d\times 1$ matrices with real entries and $\epsilon $ in $\{0,1\}$ with $\mathbb P(\epsilon=1)=p$, denote $X^\epsilon:=\epsilon X+(1-\epsilon) Y$ the mixture of $X$ and $Y$:
\begin{align}
    \mathbb E(X^\epsilon) &=\mathbb E(X) +  (1-p)(\mathbb E(Y)-\mathbb E(X)) ; \\
     \mathrm{Cov}\left( X^\epsilon,X^\epsilon \right)  &=
\mathrm{Cov}\left( X,X \right) +  (1-p) \left[\mathrm{Cov}(Y,Y)-\mathrm{Cov}(X,X) \right] + p(1-p)(\mathbb E(X)-\mathbb E(Y))(\mathbb E(X)-\mathbb E(Y))^T.
\end{align}
If $X$ follows the uniform distribution on a dataset $D'$ with $n-1$ element, $Y$ is deterministic with value at some $x_n\in \mathbb R^d$ and $p=1-1/n$ ; then $X^\epsilon$ follows the uniform distribution on the dataset $D'' = D'\cup\{x_n\}$.  Assuming $D'$ and $D''$ are in the ball of radius 1, we have:
\begin{align*}
    \|\mathbb E(X^\epsilon) - \mathbb E(X)\|_2 &= (1-p)\|(\mathbb E(Y)-\mathbb E(X))\|_2 \leq \frac{2}{n} ; \\
 \|\mathrm{Cov}\left( X^\epsilon,X^\epsilon \right) - \mathrm{Cov}\left( X,X \right)\|_F^2 &=   (1-p)^2 \|\mathrm{Cov}(X,X)\|_F^2 \\ & + p^2(1-p)^2\mathrm{Tr}\left(\mathbb E(X)-\mathbb E(Y)) (\mathbb E(X)-\mathbb E(Y))^T(\mathbb E(X)-\mathbb E(Y))(\mathbb E(X)-\mathbb E(Y))^T\right) ; \\
 &-2p(1-p)^2\mathrm{Tr}\left(\mathrm{Cov}(X,X) (\mathbb E(X)-\mathbb E(Y))(\mathbb E(X)-\mathbb E(Y))^T\right)\\
 &=(1-p)^2\|\ \mathrm{Cov}(X,X) \|_F^2+p^2(1-p)^2\|\mathbb E(X)-\mathbb E(Y)\|_2^4\\ &+2p(1-p)^2(\mathbb E(X)-\mathbb E(Y))^T \mathrm{Cov}(X,X)  (\mathbb E(X)-\mathbb E(Y)) \\ 
 &\leq \frac{1}{n^2}\|\hat{\boldsymbol{\Sigma}}\|_F^2 +16\frac{(n-1)^2}{n^4}  + \frac{8(n-1)}{n^3}\rho(\hat{\boldsymbol{\Sigma}})\\
 &\leq \frac{1}{n^2}\left(4+16+8\times 2\right) \\
  \|\mathrm{Cov}\left( X^\epsilon,X^\epsilon \right) - \mathrm{Cov}\left( X,X \right)\|_F&\leq \frac{6}{n}
\end{align*}
Where $\|\cdot\|_F$ is the Frobenius norm, $\rho$ return the largest eigenvalue. We used the bounds $\rho(\hat{\boldsymbol{\Sigma}})\leq \|\hat{\boldsymbol{\Sigma}}\|_F$ and $  \|\hat{\boldsymbol{\Sigma}}\|_F\leq 2$; the latter may be obtained in a similar fashion by developing $\mathrm{Cov}(X,X)$ using $X=\sum_{i=1}^{n-1} \eta_i X_i$ where each $X_i$ is deterministic at some $x_i$ and $\eta$ is uniform on the set of  binary vectors of length $n-1$ satisfying $\sum_i\eta_i=1$.

 Finally, the $\ell_2$-sensitivity of $(\hat{\boldsymbol{\mu}},  \hat{\boldsymbol{\Sigma}})$ is $\sqrt{\left(\frac{2}{n}\right)^2+\left(\frac{6}{n}\right)^2}=\frac{2\sqrt{10}}{n}$. Inserting the $\ell_2$-sensitivity in equation (\ref{eq:gauss}) yields the result. 
\end{proof}

\begin{remark} Note that in \cite{pmlr-v54-park17c}, they derive an $\ell_2$-sensitivity bound for $\hat{\boldsymbol{\Sigma}}$ of $2/n$ instead of our $6/n$. The reason is that one may replace $(\hat{\boldsymbol{\mu}},\hat{\boldsymbol{\Sigma}})$ by $(\hat{\boldsymbol{\mu}},\hat{\boldsymbol{\Sigma}}+\hat{\boldsymbol{\mu}} \hat{\boldsymbol{\mu}}^T)$ to reduce the $\ell_2$-sensitivity of the covariance part. This leads to the improved total $\ell_2$-sensitivity of $\frac{2\sqrt{2}}{n}$. 

On the server side,  the Gaussian mechanism would return
$(\hat{\boldsymbol{\mu}}+\Delta \boldsymbol{\mu},\hat{\boldsymbol{\Sigma}}-\hat{\boldsymbol{\mu}}\hat{\boldsymbol{\mu}}^T + \Delta\boldsymbol{\Sigma})$. The server would thus reconstruct $\hat{\boldsymbol{\Sigma}}$ by computing
\begin{align*}
    \hat{\boldsymbol{\Sigma}}_{\text{server}} &=\hat{\boldsymbol{\Sigma}} -\hat{\boldsymbol{\mu}}\hat{\boldsymbol{\mu}}^T + \Delta {\boldsymbol{\Sigma}} +    (\hat{\boldsymbol{\mu}}+\Delta\boldsymbol{\mu})(\hat{\boldsymbol{\mu}}+\Delta\boldsymbol{\mu})^T \\ 
    &= \hat{\boldsymbol{\Sigma}} + \hat{\boldsymbol{\mu}} \Delta \boldsymbol{\mu}^T + \Delta\boldsymbol{\mu} \hat{\boldsymbol{\mu}}^T +  \Delta\boldsymbol{\mu} \Delta\boldsymbol{\mu}^T + \Delta\boldsymbol{\Sigma}.
\end{align*} 
Therefore, for a given coefficient $(i,j)$ of the reconstructed covariance matrix $\hat{\boldsymbol{\Sigma}}_{\text{server}}$, the error term is $ \hat \mu_i \Delta \mu_j +  \hat \mu_j \Delta \mu_i +  \Delta \mu_i\Delta \mu_j +  \Delta \Sigma_{ij}$. Assuming a Gaussian noise of standard deviation $\sigma=\frac{2\sqrt{2}}{n}C$ with $C=\frac{\sqrt{2\ln(2/\delta)}}{\epsilon}$ for both $\Delta \boldsymbol\Sigma$ and $\Delta \boldsymbol\mu$ the standard deviation of the reconstruction error is then 
$$\sigma_{\text{reconstruction}} = \sqrt{(\hat\mu_i^2+\hat\mu_j^2) \sigma^2+\sigma^4+\sigma^2}\leq \sigma \sqrt{2+\sigma^2}\leq \frac{2\sqrt{2}C\sqrt{2+C^2/n^2}}{n}.$$

Compared to the reconstruction error of the Gaussian mechanism we used in our experiments $\sigma_{\text{reconstruction}}' = \frac{2\sqrt{10}}{n}C$, the more sophisticated methods we described above may allow better  reconstruction error without sacrificing the differential privacy for big enough datasets, ie if $n\geq \frac{\sqrt{2\ln(2/\delta)}}{\sqrt{3}\epsilon}$.

Another consequence is that the noises on different elements of the covariance matrix are not independent.

\end{remark}
\begin{remark}
The assumption of normalized features, i.e., $||f(\mathbf{x})||_2 \leq 1$, in Theorem \ref{th:dp-m=1} does not limit the performance of networks with soft-max loss function since they both have the same expressive power as shown in Proposition 3.A in \citet{zhang2023normalization}.
\end{remark}

\section{Control over 0-1 loss}\label{sec:bounds}
    Since the projection of from true feature distributions to mixture of Gaussian is lossy, one may control the consequence on the accuracy of the classifier: 
    ``Given a classifier $h:\mathcal X\rightarrow \{1,\cdots,C\}$ trained on a synthetic dataset, what guarantees do we have on the accuracy of $h$ used to classify the true dataset?".

     We adress this question by proving a theoretical bound that my be rewritten using the loss of the local client training and the accuracy of the server model.

    \subsection{A Theoretical bound}
     The setting may be formalized as follows: 
    \begin{itemize}
        \item a finite set of classes $\mathcal C = \{1,\cdots, C\}$, a feature space $\mathcal X$ and an approximate feature space $\mathcal Y$;
        \item a true feature/label distribution $\alpha$ on $\mathcal C \times \mathcal X$
        and  an approximate feature/label distribution $\beta$ on $\mathcal C \times \mathcal Y$;
        \item a classifier $h_\alpha:\mathcal X\rightarrow \mathcal C$ trained on $\alpha$ and a classifier $h_\beta:\mathcal Y \rightarrow \mathcal C$ trained on $\beta$.
        \item an feature approximator mapping $\iota:\mathcal X\rightarrow \mathcal Y$
    \end{itemize}
   
    We denote by $\mathrm {Acc}(h,\alpha)$  the accuracy  of the predictor $h$ evaluated using the distribution $\alpha$.
    
   We  prove  a bound adapted to the purpose of linking the EM loss to the accuracy. 
    \begin{theorem} Let $\alpha,\beta$ be probability distributions on $\mathcal X\times \mathcal C$ and $\mathcal Y\times \mathcal C$ with same marginal $\eta$ on $\mathcal C$. Let $h:\mathcal Y\rightarrow \mathcal C$ be any predictor and let $\iota : \mathcal X\rightarrow \mathcal Y$ be any measurable map. We have 
        $$\mathrm{Acc}(h\circ \iota ;\alpha) \geq \mathbb E_c\left[\mathrm{Acc}(h;\beta|c) \times \left(\mathrm{Acc}(h; \beta|c) -  \mathrm{div}_{TV}(\iota\#(\alpha|c), \beta|c) \right)\right] $$
    \end{theorem}
    \begin{proof}
        Consider any coupling $\mathbb P$ of $\alpha$ and $\beta$ over $\mathcal C$ ie a distribution on $\mathcal C\times\mathcal X\times \mathcal Y$ whose marginals on $\mathcal C\times \mathcal{X}$ and $\mathcal C\times \mathcal Y$ are $\alpha$ and $\beta$ respectively\footnote{ We define the conditioning of a probability distribution $\mathbb P$ by an  event $E$ having $\mathbb P(E)>0$ is 
        the distribution $\frac{\mathbf 1_{E}\mathbb P}{\mathbb P (E)}$ so that it is still a distribution on the same underlying measurable space.
        }.
        
        Given some $c$, 
        \begin{align}
            \mathrm {Acc}(h\circ \iota; \alpha|c) &= (\alpha|c)(\{x~:~ h\circ \iota(x)=c\})  \\
            &=  (\mathbb P|c)(\{(x,y)~:~ h\circ \iota(x)=c\})\\
             &\geq  (\mathbb P|c)(\{(x,y)~:~ h\circ \iota(x)=c ~\text{ and }~\iota(x)=y\})\\
             &= (\mathbb P|c)(\{ (x,y)~:~  h(y)=c\}) 
             \times  (\mathbb P|c,h(y)=c)(\{ (x,y)~:~  \iota(x)=y\})\\
              &= \underbrace{(\beta|c)(\{ y~:~  h(y)=c\}) }_{\mathrm{Acc}(h;\beta|c)}
             \times  (\mathbb P|c,h(y)=c)(\{ (x,y)~:~  \iota(x)=y\})
        \end{align}
        Since neither the left hand side of the first line nor the first term in the right hand side of the last line depend on the coupling $\mathbb P$ we chose, we may choose it to maximize  $(\mathbb P|c,h(y)=c)(\{ (x,y)~:~  \iota(x)=y\})=1-\mathbb E(\mathbf 1_{Z\neq Y})$ with $Y$ drawn from $(\beta|c,h(y)=c)$ and $Z$ drawn from $\iota\# (\alpha|c)$. We recognize a caracterization of the total variation as the minimal value of $\mathbb E(\mathbf 1_{Z\neq Y})$ over couplings of $Z$ and $Y$. Then:
        \begin{align}
          (\mathbb P|c,h(y)=c)(\{ (x,y)~:~  \iota(x)=y\})&\geq   1-\mathrm{div}_{TV}(\iota \# (\alpha|c) ~;~ (\beta|c,h(y)=c)) \\
          &\geq  1-\mathrm{div}_{TV}(\iota \# (\alpha|c) ~;~ (\beta|c))- \mathrm{div}_{TV}
          ((\beta|c) ~;~ (\beta|c,h(y)=c)) 
          \\ 
           &=1-\mathrm{div}_{TV}(\iota \# (\alpha|c) ~;~ (\beta|c))- (1-\mathrm{Acc}(h(y);\beta|c))
           \\&=\mathrm{Acc}(h(y);\beta|c)-\mathrm{div}_{TV}(\iota \# (\alpha|c) ~;~ (\beta|c)).
        \end{align}
        We use triangular inequality of total variation to get the second line. The third line is obtained by applying the general property that for any probability distribution $\mathbb Q$ and any event $E$ with $\mathbb Q(E)>0$ we have $$\mathrm{div}_{TV}(\mathbb Q,(\mathbb Q|E)) = \mathbb Q(\overline E).$$
    We thus have 
    $$ \mathrm {Acc}(h\circ \iota; \alpha|c)  \geq \mathrm{Acc}(h;\beta|c) \times \left[\mathrm{Acc}(h(y);\beta|c)-\mathrm{div}_{TV}(\iota \# (\alpha|c) ~;~ (\beta|c))\right]$$
        The result follows by taking expectation over $c$ distributed along $\eta$.
    \end{proof}

    \subsection{Application to FedPFT}
    We now switch back to the notations used in the preliminaries.
    Using Pinsker inequality \cite{csiszar2011information}, we obtain the following bound for the accuracy of the server model 
    \begin{equation}
        \mathrm{Acc}(h, F^i) \geq  \mathbb E_{c}\left[\mathrm{Acc}(h, \widetilde F^{i,c}) \times \left( \mathrm{Acc}(h, \widetilde F^{i,c}) - \sqrt{\frac{1}{2}\left(\mathrm {div}_{KL}(\widetilde F^{i,c} || F^{i,c})\right) } \right)\right].
    \end{equation}
    Since the EM algorithm maximizes the log likelyhood of the Gaussian mixture given dataset samples, we may insert
    $\mathrm{div}_{KL}(\widetilde F^{i,c} || F^{i,c})= \mathcal H^{i,c}-\mathcal L_{EM}^{i,c} $
    to obtain
\begin{equation}
        \mathrm{Acc}(h, F^i) \geq  \mathbb E_{c}\left[\mathrm{Acc}(h, \widetilde F^{i,c}) \times \left( \mathrm{Acc}(h, \widetilde F^{i,c}) - \sqrt{\frac{1}{2}\left( \mathcal H^{i,c}-\mathcal L_{EM}^{i,c})\right) } \right)\right]
    \end{equation}
    where $\mathcal H^{i,c}$ is the self-entropy of the distribution $F^{i,c}$ of features with label $c$ in client $i$ and $\mathcal L_{EM}^{i,c}$ is the log-likelyhood of the Gaussian mixture $\widetilde F^{i,c}$ trained using the EM algorithm.
    The last equation yields inequation \ref{eq:local_guarantee} by replacing accuracies by 0-1 losses.
    Beware that the feature distribution is a priori discrete, in order to evaluate the entropy we need to dequantize the dataset, otherwise, $\mathcal H^{i,c}=+\infty$ and the bound is useless. 
    
\subsection{Comparison to earlier bounds}
 Commonly known bounds on the accuracy of $h$ may be found in \citet{ben2010theory}, but these bounds are unfortunately not directly useful in our setting. Indeed, they depend on an estimation on how much  server/client class predictors differ. 
 
 More precisely, in the limit of perfect accuracy of the server class predictor, ie $\forall i,c,~\mathrm{Acc}(h, \widetilde F^{i,c})\simeq 1$,  our bound yields 
 \begin{equation}
        \mathrm{Acc}(h, F^i) \gtrsim  \mathbb E_{c}\left[ 1 - \sqrt{\frac{1}{2}\left( \mathcal H^{i,c}-\mathcal L_{EM}^{i,c})\right) } \right].
    \end{equation} 
 In this limit, the bound deduced from that of \cite{ben2010theory} contains an additional intractable negative term on the right hand side:
 \begin{equation}
        \mathrm{Acc}(h, F^i) \geq  \mathbb E_{c}\left[ 1 - \sqrt{\frac{1}{2}\left( \mathcal H^{i,c}-\mathcal L_{EM}^{i,c})\right) } \right] - \mathbb E_{c}\min \left(\mathbb E_{x\sim F^{i,c}}\left[\mathbf 1_{h(x)\neq h^*_{\text{client}}(x)}\right] ; \mathbb E_{x\sim \widetilde F^{i,c}}\mathbf 1_{h(x)\neq h^*_{\text{client}}(x)} \right)
    \end{equation}
    where $h^*_{\text{client}}$ is a hypothetical class predictor perfectly fine-tuned by the client with its own data $f(D_i)$.
However, our bound may be less sharp in general as the positive term is in $\mathrm{Acc}(h,\widetilde F^{i,c})^2$. This leads to a quick degradation of the theoretical accuracy guarantee as the measured server accuracy drops. In general, we should have error on the server side at least twice lesser than the target error on the client side in order to hope for achievable EM-loss target on client side.


\clearpage
\section{Experiments}\label{sec:exp-details}
In this section, we provide the full details of all of our experiments and datasets.

\subsection{Implementation details} 
We use the Pytorch library for the implementation of our methods. We use a cluster of 4 NVIDIA v100 GPUs for our experiments. 

\paragraph{Datasets.} We provide the complete details of the dataset we used in this paper in Table \ref{tbl:datasests_full}.

\paragraph{Learning rates and optimizers.} We use Adam with a learning rate of $1e^{-4}$ for training the classifier head for FedPFT, Ensemble, FedFT, and centralized training.

\paragraph{Knowledge Distillation.} For the implementation of KD methods, we use Adam for training the local classifier heads with $1e^{-4}$ learning rate. We locally train the models for 100 epochs and keep the best model on the local test dataset. Then, we distill from the source to the destination model in 50 epochs with the same optimizer and classifier head. We test three temperature values $\{1, 5, 10\}$ and report the best results

\paragraph{FedAvg.} We test three local training epochs $\{50, 100, 200\}$ and three learning rates $\{5e^{-2}, 1e^{-2}, 1e^{-2}\}$ and report the best result.

\paragraph{FedYogi.} We set the server learning rate $\eta$ to 0.01, $\beta_1$, $\beta_2$, and $\tau$ to 0.9, 0.99, and 0.001 respectively. We also initialize vector $v_t$ to $1e^{-6}$. For the rest of the hyperparameters, including the clients' learning rates, we use the same hyperparameters as FedAvg.

\paragraph{FedProx.} We add a regularizer with a weight of $0.01$ to the loss function of each client during local training to penalize divergence. We use the same hyperparameters as FedAvg for the rest of the hyperparameters.

\paragraph{DSFL.} We set the top-K sparsification ($K$ in the DSFL paper) as half of the number of parameters of the classifier head.

\paragraph{FedBE.} We use Adam with a learning rate of $1e^{-4}$ for training the classifier head and we sample 15 models from the posterior distribution of classifier heads.

\begin{table*}[h]
  \caption{Summary of datasets}
  \label{tbl:datasests_full}
    \centering
    {\begin{tabular}{lccccc}
    \toprule
    Dataset   & Image size      & \# Train        & \# Testing      & \#Classes & Feature extractor  \\
    \midrule
    CIFAR10   &  (32, 32)       & 50,000          & 10,000          & 10        & ResNet-50                     \\  
    CIFAR100   &  (32, 32)       & 50,000         & 10,000          & 100        & ResNet-50                     \\  
    PACS (P) &  (224, 224)       & 1,336         & 334          & 7        & Base ViT, 16                   \\  
    PACS (S) &  (224, 224)       & 3144         & 785          & 7        & Base ViT, 16                     \\
    Office Home (C) &  min:(18, 18)       & 3492         & 873          & 65        & Base ViT, 16                    \\
    Office Home (P) &  min:(75, 63)       & 3551         & 888          & 65        & Base ViT, 16                    \\
    Caltech101 &  min:(200, 200)       & 6084         & 3060          & 101        & CLIP, ViT-B/32                   \\
    Stanford Cars &  min:(360, 240)       & 12948         & 3237          & 196        & CLIP, ViT-B/32                    \\
    Oxford Pets &  min:(108, 114)       & 3680          & 3669           & 37        & CLIP, ViT-B/32                    \\
    Food101 &  max:(512, 512)       & 75750         & 25250          & 101        & CLIP, ViT-B/32                    \\
    \bottomrule
  \end{tabular}}
\end{table*}
\clearpage
\subsection{Comparing to existing one-shot and multi-shot FL
methods}
In this section, we provide the details for the experiment in Section \ref{sec:comp_sota}. Table \ref{tbl:FL-50} summarizes one-shot methods results in Table format for Figure \ref{fig:com_acc}. Further, Figure \ref{fig:data_partition_caltech} and \ref{fig:data_partition_cifar100} shows the data partitions based on Dirichlet shift with $\beta=0.1$ with 50 clients for Caltech101 and CIFAR100 Datasets, respectively. The magnitude of data samples for each class label in each client is represented by the size of the red circle. This figure shows the non-iidness of data distribution among clients.

\begin{table*}[h]
\centering
  \caption{one-shot methods results in Table format for Figure \ref{fig:com_acc}}
  \label{tbl:FL-50}
    \begin{tabular}{lcccc}
    \toprule
    & \multicolumn{2}{c}{CIFAR100 ($\beta=0.1$)}  & \multicolumn{2}{c}{Caltech101 ($\beta=0.1$)}  \\
    Methods   & Accuracy      & Comm.       & Accuracy      & Comm.  \\
    \midrule
    \tr Centralized  & \tr 73.50 $\pm$ \tiny{0.04}              & \tr 97 MB            & \tr 95.36 $\pm$ \tiny{0.08}               & 3.6 MB         \\
    Ensemble       & 53.97 $\pm$ \tiny{0.49}              & 19 MB            & 72.86 $\pm$ \tiny{2.09}               & 4.9 MB         \\     
    AVG          & 43.51 $\pm$ \tiny{0.26}              & 19 MB            & 56.90 $\pm$ \tiny{0.67}             & 4.9 MB         \\     
    \midrule   
    FedKT          & 51.80 $\pm$ \tiny{0.02}              & 19 MB            & 46.64 $\pm$ \tiny{0.33}             & 4.9 MB         \\
    FedBE          & 56.32 $\pm$ \tiny{0.44}              & 19 MB            & 57.67 $\pm$ \tiny{0.27}             & 4.9 MB         \\
    \midrule
    $\hat{\mathcal{G}}_{\text{spher}}$(k=1)      & 56.96 $\pm$ \tiny{0.23}             & 22 MB          & 86.41 $\pm$ \tiny{0.13}              & 3 MB       \\    
    \midrule 
    $\mathcal{G}_{\text{diag}}$(k=10)      & 69.91 $\pm$ \tiny{0.11}             & 0.3 GB          & 94.67 $\pm$ \tiny{0.03}              & 23 MB       \\     
    $\mathcal{G}_{\text{diag}}$(k=50)      & 72.05 $\pm$ \tiny{0.10}             & 0.9 GB          & 94.74 $\pm$ \tiny{0.13}              & 33 MB       \\     
    $\mathcal{G}_{\text{diag}}$(k=100)      & \bf{72.29} $\pm$ \tiny{0.08}             & 1.2 GB          & \bf{94.83} $\pm$ \tiny{0.18}              & 36 MB       \\      
    \midrule 
    $\mathcal{G}_{\text{spher}}$(k=1)      & 60.16 $\pm$ \tiny{0.04}             & 22 MB          & 93.46 $\pm$ \tiny{0.16}              & 3 MB       \\     
    $\mathcal{G}_{\text{spher}}$(k=10)      & 69.74 $\pm$ \tiny{0.16}             & 0.2 GB          & 94.59 $\pm$ \tiny{0.13}              & 12 MB       \\     
    $\mathcal{G}_{\text{spher}}$(k=50)      & \bf{71.88} $\pm$ \tiny{0.03}             & 0.5 GB          & \bf{94.71} $\pm$ \tiny{0.16}              & 16 MB       \\    
    
    \bottomrule
  \end{tabular}
\end{table*}
\clearpage
\begin{figure}[t!]
\centering 
\includegraphics[width=0.99\columnwidth]{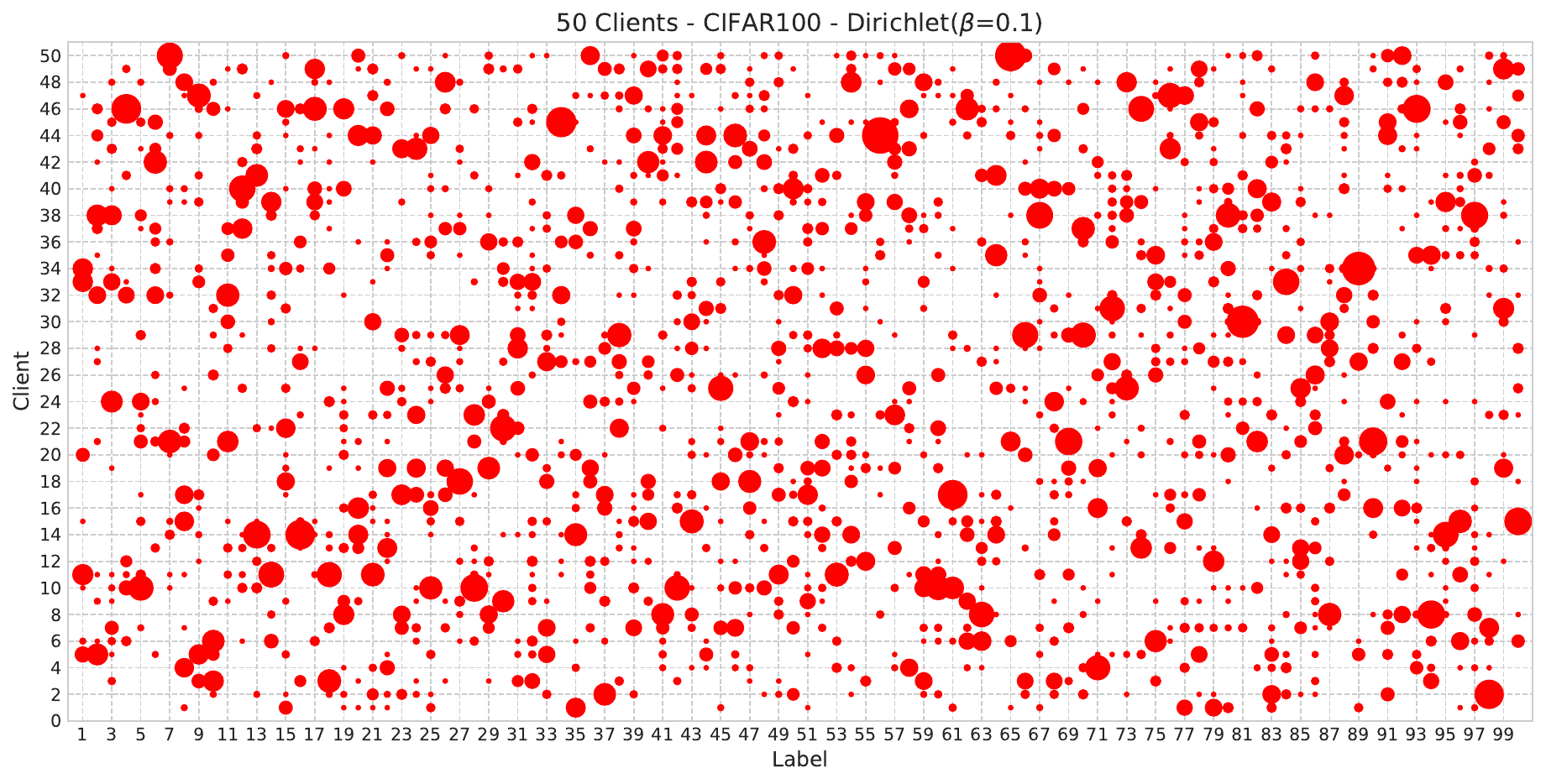} %
\caption{Data partitions based on Dirichlet shift with $\beta=0.1$ with 50 clients for CIFAR100 datasets. The size of the red circle represents the magnitude of data samples for each class label in each client}
\label{fig:data_partition_cifar100}
\end{figure}

\begin{figure}[t!]
\centering 
\includegraphics[width=0.99\columnwidth]{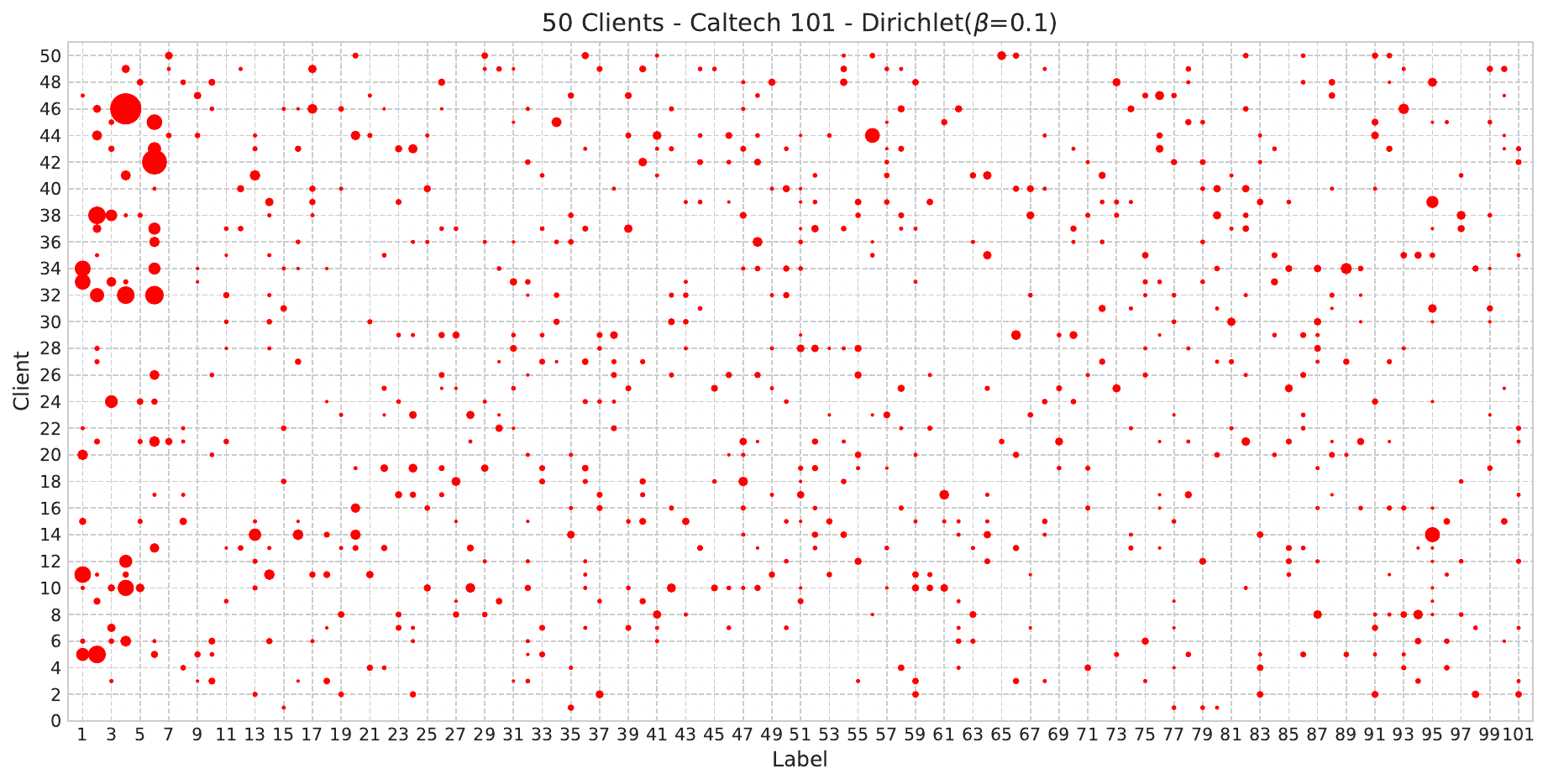} 

\caption{Data partitions based on Dirichlet shift with $\beta=0.1$ with 50 clients for Caltech101 datasets. The magnitude of data samples for each class label in each client is represented by the size of the red circle}
\label{fig:data_partition_caltech}
\end{figure}

\clearpage
\section{Reconstruction attack details}\label{sec:recon-details}

We conduct reconstruction attacks on the various feature-sharing schemes described in this paper. First, we verify the vulnerability of raw feature sharing: we demonstrate that an attacker with access to in-distribution data can obtain high-fidelity reconstructions of private training data. This necessitates sharing schemes beyond sending raw features. The attack we present proceeds by training a generative model on (feature, image) pairs, and then performing inference on received feature embeddings. Next, we apply the aforementioned attack on features sampled from GMMs received via our proposed scheme. We find that the resulting reconstructions do not resemble real training data.


\subsection{Threat model}
We consider the setting where two clients (a \emph{defender} and an \emph{attacker}) collaborate via a feature-sharing scheme to train a model to perform well on the union of their datasets. \emph{Both parties} have:
\begin{itemize}
    \item Local datasets of (image, label) pairs, which we call $D_d$ and $D_a$ respectively, where $D_* = \{(x_i,y_i)\}_{i=1}^{m_*}$.
    \item Black-box access to a feature embedding function $E : \mathcal X \to \mathcal Z$.
\end{itemize}

\emph{The defender} produces an embedded version of their dataset $E(D_d):= \{(E(x_i), y_i)\}_{i=1}^{m_d}$, and passes it along to the attacker via a feature-sharing scheme. We consider sending: the embedded dataset $E(D_d)$ directly (\emph{raw features}), our proposedfeature-sharing scheme (\emph{FedPFT}), as well as with differentially privacy (\emph{DP-FedPFT}).

\emph{The attacker} has a dataset $D_a$, which is assumed to be in-distribution of the defender's dataset $D_d$. Concretely: $D_a$ is the CIFAR-10 train set ($50$K examples) and $D_d$ is the CIFAR-10 test set ($10$K examples) in our experiments.

\subsection{Attacker objectives.} We identify 3 attacker objectives representing varying levels of privacy violation, ordered by strictly decreasing attack strength.
\begin{enumerate}
    \item \emph{(Total reconstruction).} The attacker is able to accurately reconstruct every example in $E(D_d)$.
    \item \emph{(Partial reconstruction).} The attacker is able to identify a small subset of $E(D_d)$, on which it can accurately reconstruct.
    \item \emph{(Set-level reconstruction).} The attacker is able to produce a set of reconstructions and a small subset of them correspond to real training points.
\end{enumerate}

Although a weaker attack than (1), (2) still constitutes a strong privacy violation, since privacy is a \emph{worst-case} notion: a priori, a user submitting their data does not know whether they are part of the reconstructable set or not (see \citet{lira} for discussion). Success in (3) would imply an attacker can generate accurate reconstructions, but are unable to identify which candidates correspond to real training data. Note that success in (3) combined with a good membership inference attack implies success in (2).


\subsection{Experimental details} 
We train U-Net-based conditional denoising diffusion models on the CIFAR-10 train set. During training, extracted features are added to time-step embeddings and condition the model via shift-and-scale operations inside normalization layers. We sample with DDIM and classifier-free guidance.

We find that there is a non-trivial overlap between the CIFAR-10 test and train sets; on which our models memorize and reconstruct perfectly. To account for this, we filter the CIFAR-10 test set for near-duplicates: we remove the 1K test images with the highest SSIM score with a member of the train set, leaving 9K images to evaluate on. 
We also manually inspect reconstructions and verify that they differ from the closest training image to the target.

\subsection{Results}
\paragraph{Raw feature reconstruction.} We present the results of our reconstruction attack on raw features in Figure \ref{fig:raw-feature-recon} for 3 selection methods (\emph{all}, \emph{attacker}, and \emph{oracle}). Each selection method is representative of performance in the corresponding attacker objectives of \emph{total}, \emph{partial}, and \emph{set-level} reconstruction. Quantitative image similarity measures are reported in Table \ref{tab:raw-feature-recon}. 

For \emph{attacker-selection}, we sample 10 reconstructions from each embedding, compute the average pairwise SSIM amongst our reconstruction, and select the top $1\%$ of reconstructions according to this metric. This is based on the intuition that stronger determinism during sampling implies the model is more confident about what the information in the embedding corresponds to.

\begin{table}[!h]
    \centering
    \begin{tabular}{clrrr}
    \toprule
    \multirow{2}{*}{\textbf{Method}} & \multirow{2}{*}{\textbf{Selection}} & \multicolumn{3}{c}{\textbf{Similarity measure}} \\
    \cmidrule{3-5}
     &  & \multirow{1}{*}{\emph{PSNR $\uparrow$}} & \multicolumn{1}{c}{\emph{LPIPS $\downarrow$}} & \multicolumn{1}{c}{\emph{SSIM $\uparrow$}} \\
    \midrule
    \multirow{4}{*}{\shortstack{ResNet-50 \\ reconstruction}} & All & 13.5 & .384 & .168 \\
                                                           \cmidrule(lr){2-5}
                                                         & Attacker & 14.7 & .297 & .358  \\
                                                           \cmidrule(lr){2-5}
                                                         & Oracle & 16.5 & .257 & .535  \\
    \midrule
    \multirow{4}{*}{\shortstack{MAE \\ reconstruction}}  & All & 15.6 & .305 & .276  \\
                                                      \cmidrule(lr){2-5}
                                                    & Attacker & 17.9 & .191 & .579  \\
                                                      \cmidrule(lr){2-5}
                                                    & Oracle & 19.5 & .181 & .674  \\
    \midrule
    \multirow{1}{*}{Train set} & Oracle & 14.7 & .389 & .500  \\
    \bottomrule
    \end{tabular}
    \caption{Similarity metrics between original image and \textbf{raw feature reconstructions} on the filtered CIFAR-10 test set. Figures are computed on different selections of the test set, representing different attacker objectives. \textbf{All:} averaged result over entire test set; performance corresponds to \emph{total reconstruction} objective. \textbf{Attacker:} attacker selects 1$\%$ ($n=90$) reconstructions without access to ground truth; corresponds to \emph{partial reconstruction} objective. \textbf{Oracle:} top 1$\%$ ($n=90$) reconstructions selected based on ground-truth SSIM; corresponds to \emph{set-level reconstruction}.}\label{tab:raw-feature-recon}
\end{table}

\begin{figure}[!h]
    \centering 
    \includegraphics[width=0.8\columnwidth, trim={0cm 8.2cm 11.2cm 3.1cm}, clip]{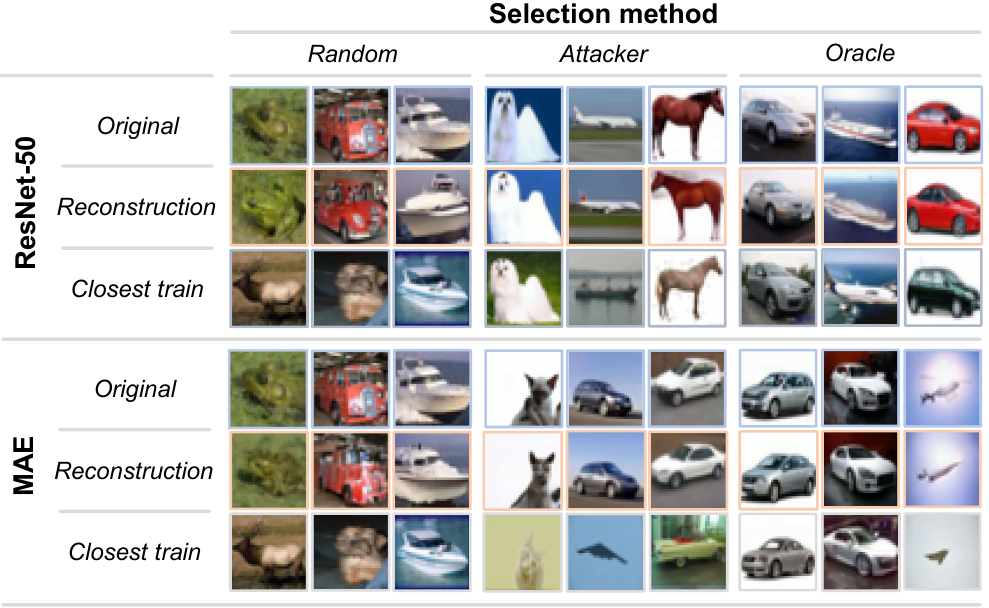} 
    \caption{Results comparing the original image and \textbf{raw feature reconstructions}, along with the original images' closest train set member by SSIM. We present results for 2 backbones: ResNet-50 and Masked Autoencoder; and 3 selection methods (\emph{random}, \emph{attacker}, and \emph{oracle}), which correspond to performance on \emph{total}, \emph{partial}, and \emph{set-level} reconstruction.}
    \label{fig:raw-feature-recon}
\end{figure}

Our main result is that an attacker is capable of producing reasonable reconstructions in all 3 settings, outperforming a baseline of selecting the closest training image. Furthermore, we find that: (1) The feature backbone affects reconstruction quality (MAE reconstructions are better than ResNet-50); (2) The attacker can effectively employs heuristics (intra-sample similarity) to identify which reconstructions are likely to be good, approaching the results of oracle selection based on the ground-truth image.

\paragraph{FedPFT reconstruction.} Table \ref{tab:gmm-recon} and Figure \ref{fig:gmm-recon} show results for set-level reconstruction of GMM-sampled features, for both random and worst-case test images. We see that even when with a ground truth similarity oracle, most images fail to be reconstructed (\emph{Oracle-Random} column of Figure \ref{fig:gmm-recon}). DP further diminishes attack effectiveness, in particular for worst-case set-level reconstructions.

\begin{table}[!h]
    \centering
    \begin{tabular}{cclrrr}
    \toprule
    \multirow{2}{*}{\textbf{Backbone}} & \multirow{2}{*}{$\epsilon$} & \multirow{2}{*}{\textbf{Selection}} & \multicolumn{3}{c}{\textbf{Similarity measure}} \\
    \cmidrule{4-6}
     &  & & \multirow{1}{*}{\emph{PSNR $\uparrow$}} & \multicolumn{1}{c}{\emph{LPIPS $\downarrow$}} & \multicolumn{1}{c}{\emph{SSIM $\uparrow$}} \\
    \midrule
    \multirow{5}{*}{ResNet-50} & \multirow{2}{*}{$\infty$} & Oracle & 14.7 & .423 &  .508 \\
                                                           \cmidrule(lr){3-6}
                                                         & & Oracle-all & 12.7 & .546 & .326  \\
                                                           \cmidrule(lr){2-6}
                            & \multirow{2}{*}{$10$} & Oracle & 14.6 & .511 & .483 \\
                                                           \cmidrule(lr){3-6}
                                                         & & Oracle-all & 12.7 & .571 & .322  \\
    \midrule
    \multirow{5}{*}{MAE} & \multirow{2}{*}{$\infty$} & Oracle & 14.5 & .456 & .497 \\
                                                       \cmidrule(lr){3-6}
                                                     & & Oracle-all & 12.5 & .558 & .323 \\
                                                       \cmidrule(lr){2-6}
                        & \multirow{2}{*}{$10$} &  Oracle & 13.4 & .553 & .452 \\
                                                       \cmidrule(lr){3-6}
                                                     & & Oracle-all & 12.2 & .595 & .306  \\

    \bottomrule
    \end{tabular}
    \caption{ Similarity metrics between original images and \textbf{set-level reconstructions} on the filtered CIFAR-10 test set. For each test image, we match it to its closest in terms of SSIM among all reconstructions. \textbf{Oracle:} the top $1\%$ $(n=90)$ of matched reconstructions by ground-truth SSIM, corresponding to performance on \emph{set-level reconstruction}. \textbf{Oracle-all:} average similarity between test images and their matched pairs.
    }\label{tab:gmm-recon}
\end{table}

\begin{figure}[!h]
    \centering 
    \includegraphics[width=0.7\columnwidth, trim={0cm 19.7cm 30.7cm 0cm}, clip]{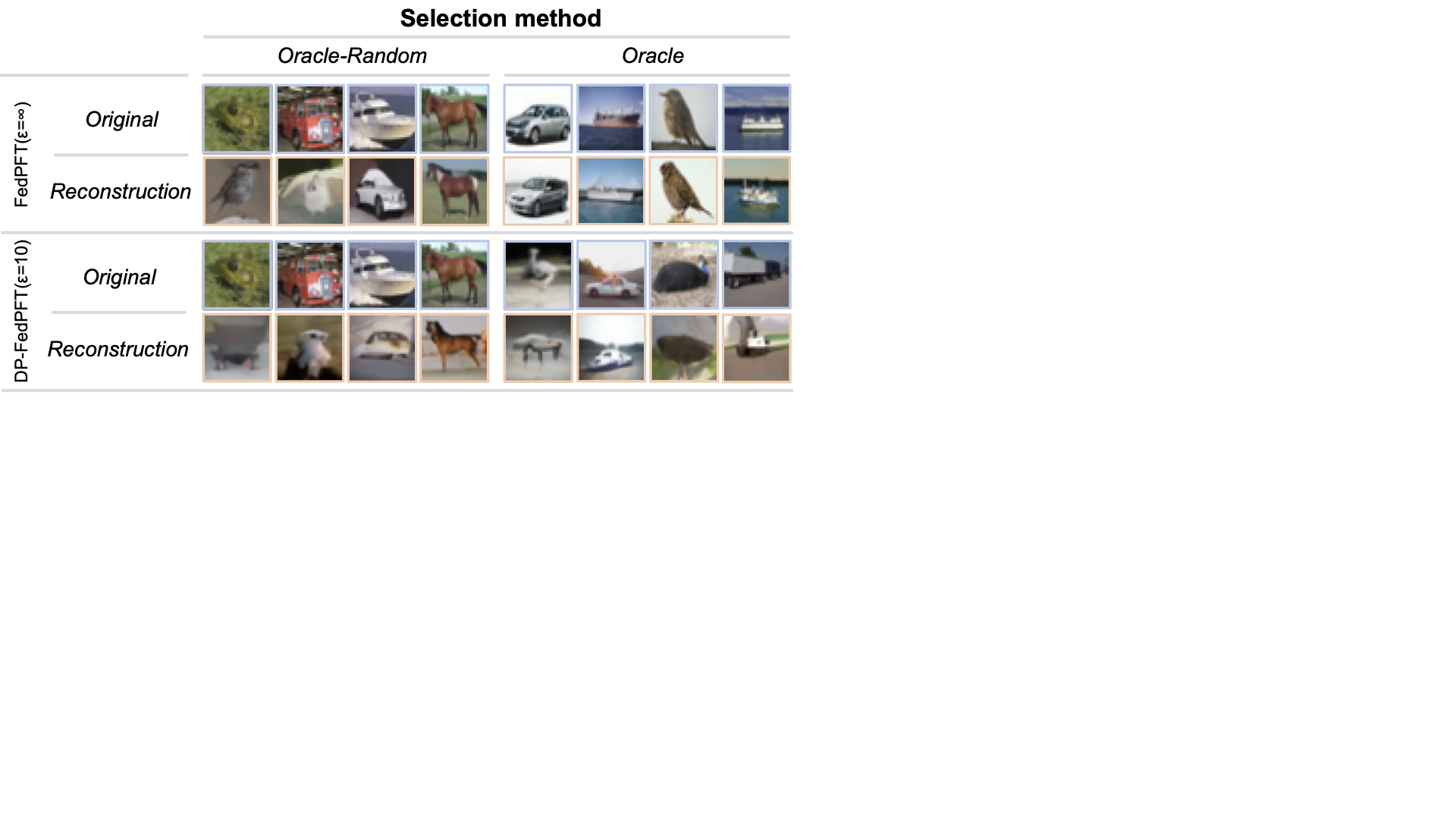}
    \caption{ Results comparing the original image and \textbf{set-level reconstructions} with ResNet-50 backbone. We present results for two selection methods. \textbf{Oracle-Random:} the closest image in the reconstruction set by SSIM for random test images. \textbf{Oracle:} the closest image in the reconstruction set by SSIM for worst-case test images.}
    \label{fig:gmm-recon}
\end{figure}


\end{document}